\documentclass[11pt]{article}
\usepackage[utf8]{inputenc} 
\usepackage[T1]{fontenc}    
\usepackage{hyperref}       
\usepackage{url}            
\usepackage{booktabs}       
\usepackage{amsfonts}       
\usepackage{nicefrac}       
\usepackage{microtype}      
\usepackage{xcolor}         
\usepackage{amsmath,amsthm, bbm} 
\usepackage{graphicx}
\usepackage{subfigure}
\usepackage{enumitem}
\setlist[itemize]{leftmargin=1cm}
\setlist[enumerate]{leftmargin=1cm}
\usepackage{enumitem}
\usepackage{booktabs}
\usepackage{multirow}
\usepackage{adjustbox}
\usepackage{algorithm}
\usepackage{algorithmic}
\usepackage{wrapfig}
\usepackage{comment}
\usepackage{multirow}
\usepackage{booktabs}
\usepackage{multirow}
\usepackage{adjustbox}
\usepackage{pdflscape}
\usepackage{pifont}
\usepackage{array}
\usepackage{natbib}
\usepackage{graphicx}
\usepackage{subcaption}
\usepackage{fontawesome5} 

\usepackage[margin = 1in]{geometry}

\newtheorem{theorem}{Theorem}[section]
\newtheorem{proposition}[theorem]{Proposition}

\theoremstyle{definition}

\newtheorem{assumption}[theorem]{Assumption}
\theoremstyle{remark}
\newtheorem{remark}[theorem]{Remark}

\newtheoremstyle{named}{}{}{\itshape}{}{\bfseries}{.}{.5em}{\thmnote{#3 }#1}
\theoremstyle{named}

\newcommand{\Pro}{\mathbb{P}}

\newcommand{\lx}[1]{{\color{purple}{[\bf\sf lx: #1]}}}

\newcommand{\bx}{\mathbf{x}}

\newcommand{\bc}{\mathbf{c}}

\newcommand{\bs}{\mathbf{s}}

\newcommand{\obs}{\mathrm{obs}}
\newcommand{\true}{\text{true}}

\newcommand{\calX}{\mathcal{X}}
\newcommand{\barcalX}{\bar{\mathcal{X}}}
\newcommand{\na}{\mathrm{na}}

\newcommand*{\ie}{{\it i.e.}{}}

\newcommand{\algname}{Diff-Joint}
\newcommand{\algnamefd}{Forest-Diffusion}
\newcommand{\algnamekmean}{$k$-means clustering with $k=2$}

\newcommand{\cmark}{\ding{51}}
\newcommand{\xmark}{\ding{55}}

\hypersetup{
    colorlinks=true,
    citecolor={blue!50!black},
    urlcolor={cyan!40!black},
    linkcolor={red!60!black}
}

\usepackage{comment}

\title{\bfseries\LARGE Learning What Not to Impute: An Uncertainty-Aware Diffusion Framework for Meaningful Missingness}

\date{}

\begin{document}

\maketitle

\vspace{-5em}

\begin{center}
{\large
Lixing Zhang$^{1}$ \qquad
Yidong Ouyang$^{2}$ \qquad
Weifu Li$^{1}$ \qquad
Shixiang Zhu$^{3}$ \\
Guang Cheng$^{2}$ \qquad
Liyan Xie$^{1}$
}

{\large
$^{1}$University of Minnesota
\quad
$^{2}$University of California, Los Angeles
\\
$^{3}$Carnegie Mellon University
}

\vspace{0.8em}

\faGithub\;Code: \href{https://github.com/lxzhang1/Diff-Joint}
{https://github.com/lxzhang1/Diff-Joint}

\begingroup
\renewcommand\thefootnote{}
\footnotetext{\it \hspace*{-1.8em}Main contact: liyanxie@umn.edu
}
\endgroup

\vspace{0.5em}

\end{center}

\begin{abstract}
Missing value imputation is a fundamental task in machine learning, with most existing methods assuming that all missing entries correspond to unobserved regular values. In many real-world datasets, however, missingness may arise from two distinct sources: some entries are {\it meaningfully missing} (intrinsically absent and semantically valid), while others are missing due to the observation process and should be imputed. We formalize this distinction as a {\it selective imputation} problem, where the goal is to jointly infer which missing entries should be preserved and which should be recovered. To address this challenge, we propose \algname{}, a diffusion-based framework that jointly models tabular data together with a latent missingness mask. The method alternates between conditional sampling and uncertainty-aware aggregation to iteratively refine both imputed values and missingness labels. Empirical results on synthetic and real-world datasets demonstrate that \algname{} effectively identifies meaningfully missing entries while achieving competitive imputation accuracy and improved downstream task performance.
\end{abstract}

\section{Introduction}

Missing value imputation is a common and important problem in machine learning, statistics, and data mining \cite{donders2006gentle,emmanuel2021survey}. In many applications, training datasets contain missing entries that must be imputed either as quantities of direct interest or as a preprocessing step for
downstream modeling, analysis, and decision-making. 
A broad range of methods have been developed to recover missing entries from the observed data values, ranging from classical statistical procedures \cite{dempster1977em,rubin1988overview,vanbuuren2011mice, stekhoven2012missforest} to modern deep generative models \cite{yoon2018gain,mattei2019miwae,tashiro2021csdi, zheng2022tabcsdi}.

Under the conventional imputation task, all missing entries are typically treated as unobserved regular values that need to be fully recovered. However, it is important to notice that in some datasets, the missing value can arise from two distinct sources: an entry may be meaningfully missing in the complete record, or a regular value that went missing during the observation process \cite{rubin1976inference}. For example, in clinical records, a laboratory measurement may be missing because the test was not ordered, but the absence of the test order itself may carry clinical meaning \cite{lipton2015learning}. In survey data, a response such as ``n/a'' may be a valid answer rather than an unobserved value; for instance, income from employment is genuinely not applicable for a respondent who is not employed \cite{allison2009missing}. In e-commerce or recommender-system data, the absence of a product attribute may indicate that the attribute is inapplicable to the item, rather than that the attribute was accidentally omitted \cite{hu2008collaborative, rendle2012bpr}. See Figure~\ref{fig:mm-example} for a conceptual illustration.

\begin{wrapfigure}{r}{0.4\textwidth}
  \begin{center}
  \vspace{-0.1in}
    \includegraphics[width=0.95\linewidth]{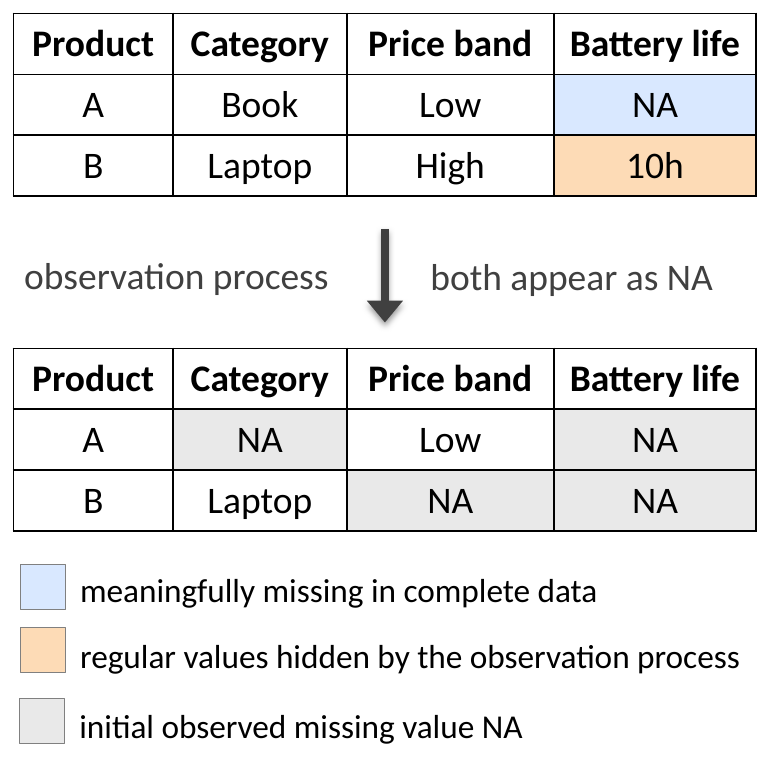}
  \end{center}
  \vspace{-0.1in}
  \caption{An illustrative example of ``meaningful missingness.'' 
  }\label{fig:mm-example}
  \vspace{-0.1in}
\end{wrapfigure}

This distinction implies that not all observed missing entries should be treated in the same way. Some entries are missing because the missing state itself is part of the underlying record. When a missing entry is meaningful, imputing it with a regular value can distort the data distribution, remove useful semantic information, and introduce bias to the learned data distribution. 
The central goal is therefore not only to estimate the value of a missing
entry, but also to determine whether the entry should be imputed. This problem is challenging because the two sources of missingness are not directly labeled in the observed data. Moreover, the distinction is entry-wise and context-dependent: the same missing token may be meaningful for
one sample but observation-induced for another. 


This calls for new imputation methods that can decide not only {\it how to impute}, but also {\it when not to impute}. In particular, the goal is to preserve missing entries that are meaningful while imputing only those entries that are missing due to the observation process. To this end, we propose a new learning framework that performs {\it selective imputation}: rather than completing all missing entries, the method jointly identifies meaningfully missing states and recovers observation-induced missing values.
We instantiate this framework using diffusion models \cite{ho2020denoising,Song2021ScoreBasedGM}, which provide a flexible backbone for modeling complex tabular distributions and imputation \cite{zheng2022tabcsdi,zhang2025diffputer}. 

The proposed method, \algname{}, introduces a joint diffusion state $(\bx,\bc)$, where $\bx$ represents the completed tabular values and $\bc$ is a binary mask indicating
which entries are meaningfully missing. Starting from a random initialization, \algname{} alternates between two steps. First, it trains a diffusion model on the current joint state. Second, it draws multiple conditional samples given the observed entries and aggregates these samples to update both the imputed values and the meaningfully-missing mask. The aggregation step leverages the entrywise {\it uncertainty scores}: high posterior uncertainty provides evidence that an observed ``$\na$'' may correspond to a meaningful missing state rather than a recoverable regular value. Through this iterative refinement, the model learns both the data distribution and the missingness structure.

Our contributions are summarized as follows.
\begin{enumerate}
    \item We formulate a selective imputation problem for tabular data with two latent sources of observed missingness: meaningfully missing entries that should be preserved, and observation-induced missing entries that should be imputed.

    \item We propose \algname{}, a diffusion-based framework that jointly models completed tabular values and meaningfully-missing masks. By alternating between diffusion-model training and uncertainty-aware aggregation of conditional samples, \algname{} iteratively refines both the imputed values and the labels of meaningfully missing entries.
    
    \item We evaluate the proposed method on both a synthetic Bayesian-network dataset and a real-world dataset based on MIMIC-IV-ED. The results show that \algname{} can identify meaningfully missing entries while maintaining competitive imputation performance and improving downstream predictive performance.
\end{enumerate}

\vspace{-0.1in}
\paragraph{Related Work.}
Missing-value imputation has been extensively studied, and existing approaches can be roughly divided into two categories: classical statistical methods and modern deep generative models.  
Classical missing-data theory separates the underlying data distribution from the observation process, and commonly categorizes missingness mechanisms as missing completely at random (MCAR), missing at random (MAR), or missing not at random (MNAR) \citep{rubin1976inference,little2019statistical}. This perspective underlies many classical imputation methods, including likelihood-based estimation with the EM algorithm \citep{dempster1977em}, chained-equation imputation \citep{vanbuuren2011mice}, and random-forest imputation for mixed-type data \citep{stekhoven2012missforest}. These methods mainly treat missing entries as unobserved regular values to be recovered. 

Deep generative models now provide a flexible framework for imputation. Representative approaches include adversarial imputation with GAIN \citep{yoon2018gain}, latent-variable modeling with MIWAE \citep{mattei2019miwae}, adaptive iterative imputation with HyperImpute \citep{pmlr-v162-jarrett22a}, and masked-reconstruction methods such as ReMasker and CACTI \citep{du2024remasker,gorla2025cacti}. 
Diffusion and score-based generative models have become a powerful class of generative models \citep{ho2020denoising,ouyang2023missdiff,Song2021ScoreBasedGM,suh2025timeautodiff}. For missing-data problems, CSDI trains conditional score-based models for probabilistic time-series imputation, while TabCSDI adapts this idea to mixed-type tabular data \citep{tashiro2021csdi,zheng2022tabcsdi}. Forest-Diffusion combines diffusion or flow-based generative modeling with gradient-boosted trees for tabular generation and imputation \citep{jolicoeurmartineau2023generating}. Closest to our iterative training procedure, DiffPuter combines diffusion models with an EM-style refinement loop to learn from incomplete data and update missing-value estimates through conditional sampling \citep{zhang2025diffputer}. These diffusion-based methods mainly focus on completing all missing values in the data rather than modeling the meaningful-missingness explicitly.

Another related line of work recognizes that missingness patterns themselves can carry useful information. In this direction, missing-data handling has been studied for robust prediction from incomplete inputs \citep{garcia2010pattern}, and recurrent models for clinical time series incorporate masks and time gaps as predictive features \citep{che2018recurrent}. Recent work on synthetic data generation also emphasizes that preserving missingness distributions can be important for downstream utility \citep{wang2023preserving}. These works show that missingness should not always be ignored or naively filled. However, they typically use the observed missingness pattern as an auxiliary feature and do not explicitly model the entry-wise distinction between missing states. Our work makes this distinction explicit by introducing a meaningfully-missing mask and learning it jointly with the underlying data distribution.

\section{Problem Setup and Preliminaries}

We consider mixed-type tabular data with $d$ features (columns). For the $j$-th feature, let $\mathcal{X}_j$ denote its domain of regular values, and define the augmented domain $\bar{\mathcal{X}}_j = \mathcal{X}_j \cup \{\na\}$, where $\na$ is used throughout the paper to denote the observed missing values. We represent a complete data point as $\mathbf{x}^{\text{true}} = (x_1,\ldots,x_d) \in \bar{\mathcal{X}}_1 \times \cdots \times \bar{\mathcal{X}}_d$, drawn from the underlying data-generating distribution on the augmented domain. If $x_j=\na$ in $\bx^{\true}$, then the $j$-th entry is \emph{meaningfully missing} (MM), meaning that the value is intrinsically absent and should be treated as a valid state. 
To model such meaningfully missing entries, we associate each complete data instance $\bx^{\text{true}}$ with a binary mask
\begin{equation}\label{eq:mm}
\mathbf{c} = (c_1,\ldots,c_d) \in \{0,1\}^d, \quad c_j = \mathbbm{1}\{x^{\text{true}}_j = \na\}.
\end{equation}
In particular, $c_j = 1$ indicates that the $j$-th entry is intrinsically missing.

In addition to such intrinsically missing entries, an observed record may also exhibit {\it observation-induced missingness}, arising from the observation process after the underlying data instance is generated. Typical causes include incomplete data collection, recording errors, and data transfer failures. These observation-induced missing entries are also recorded using the same symbol $\na$ in the observed data. We encode this second source using another binary mask $\mathbf{r} = (r_1,\ldots,r_d) \in \{0,1\}^d$, where $r_j = 1$ indicates that the $j$-th entry is missing due to the observation process.

The final observation is then determined jointly by meaningful missingness and observation-induced missingness. We define the observation mask $\boldsymbol{\omega} = (\omega_1,\ldots,\omega_d) \in \{0,1\}^d$ with $\omega_j = (1 - c_j)(1 - r_j)$. Given $(\bx^{\text{true}}, \boldsymbol{\omega})$, the observed record $\mathbf{x}^{\mathrm{obs}} = (x_1^{\mathrm{obs}},\ldots,x_d^{\mathrm{obs}})$ is defined as $x_j^{\mathrm{obs}} = x_j$ if $\omega_j = 1$, and $x_j^{\mathrm{obs}} = \na$ otherwise. 
In other words, the same symbol $\na$ is observed regardless of whether the underlying cause is meaningful missingness or observation-induced missingness. Consequently, these two types of missingness are not directly distinguishable from the observed data alone. 
Our framework therefore assumes that the two missingness mechanisms induce distinct statistical properties in the conditional distribution of missing entries given the observed context. We formalize this mechanism-separation perspective and establish identifiability conditions in Section~\ref{sec:identifiability_theory}.

Given $n$ observed records $\bx^\obs_1,\ldots,\bx^\obs_n$, our goal is to infer which entries are meaningfully missing and to impute {\it only} those entries that are observation-induced missing. Specifically, we first aim to infer the meaningful-missingness mask
$\widehat\bc_i=(\widehat c_{i,1},\ldots,\widehat c_{i,d})$ to determine which entries of each record should remain $\na$. Second, for entries not identified as meaningfully
missing, we estimate their regular values and output an imputed value
$\widehat{x}_{i,j}\in\calX_j$. Thus, the final imputed record preserves $\na$ for entries with $\widehat c_{i,j}=1$, while replacing observation-induced missing entries with their imputed values
$\widehat{x}_{i,j}$.





\subsection{Preliminaries: Diffusion Models}\label{sec:diffusion}

This work uses diffusion models to learn the data distribution \cite{ho2020denoising}. Diffusion models are characterized by their forward and backward processes. The forward process perturbs the data distribution $p(\mathbf{x})$ by injecting Gaussian noise, as described by the following continuous-time equation \cite{Song2021ScoreBasedGM}:
\begin{equation}\label{eq:forward}
    \mathrm{d} \mathbf{x}_t=\mathbf{f}(\mathbf{x}_t, t) \mathrm{d} t+g(t) \mathrm{d} \mathbf{w}, \ t\in[0,T],
\end{equation}
where $\mathbf{w}$ is the standard Brownian motion, $\mathbf{f}(\cdot, t): \mathbb{R}^d \rightarrow \mathbb{R}^d$ is a drift coefficient, and $g(\cdot): \mathbb{R} \rightarrow \mathbb{R}$ is a  diffusion coefficient. The marginal distribution of $\mathbf{x}_t$ at time $t$ is denoted as $p_t(\mathbf{x}_t)$, and $p_0$ is the distribution of the initial value $\mathbf{x}_0$, which equals the true data distribution. 
Then, we can reverse the forward process \eqref{eq:forward} for generation, defined as:\begin{equation}\label{eq:backward}
\mathrm{d} \mathbf{x}_t=\left[\mathbf{f}(\mathbf{x}_t, t)-g(t)^2 \nabla_{\mathbf{x}} \log p_t(\mathbf{x})\right] \mathrm{d} t+g(t) \mathrm{d} \overline{\mathbf{w}},
\end{equation}
where $\overline{\mathbf{w}}$ is a standard Brownian motion when time flows backwards from $T$ to 0. 
The key of the backward process is estimating the score function of each marginal distribution, $\nabla_{\mathbf{x}} \log p_t(\mathbf{x})$, by training a score network $\mathbf{s}_{\boldsymbol{\theta}}(\mathbf{x}_t, t)$ \citep{Hyvrinen2005EstimationON,vincent2011connection,song2020sliced}
\begin{equation}\label{eq:dsm}
\boldsymbol{\theta}^*= \underset{\boldsymbol{\theta}}{\arg \min } \ \mathbb{E}_{t\sim \text{Unif}[0,T]}\left\{\lambda(t) \mathbb{E}_{p_t(\mathbf{x}_t)} \left[\left\|\mathbf{s}_{\boldsymbol{\theta}}(\mathbf{x}_t, t)-
\nabla_{\mathbf{x}_t} \log p_t(\mathbf{x}_t)\right\|_2^2\right]\right\},
\end{equation}
where $\lambda(t):[0, T] \rightarrow \mathbb{R}_{>0}$ is a positive weighting function. 


It is worthwhile mentioning that since tabular data may contain both continuous and discrete variables, we use an encoder described in Appendix~\ref{app:imple-detail} to map each record into a continuous model space for diffusion-model training and sampling.
The corresponding decoder maps generated samples back to the original mixed-type tabular space.
For notational simplicity, we use $\mathbf{x}$ for both the original and encoded representations in the rest of the paper; the distinction is clear from context. 

\section{Methodology} \label{sub:diffusion}

\begin{figure}[t]
\centering
\vspace{-0.1in}
\includegraphics[width=\linewidth]{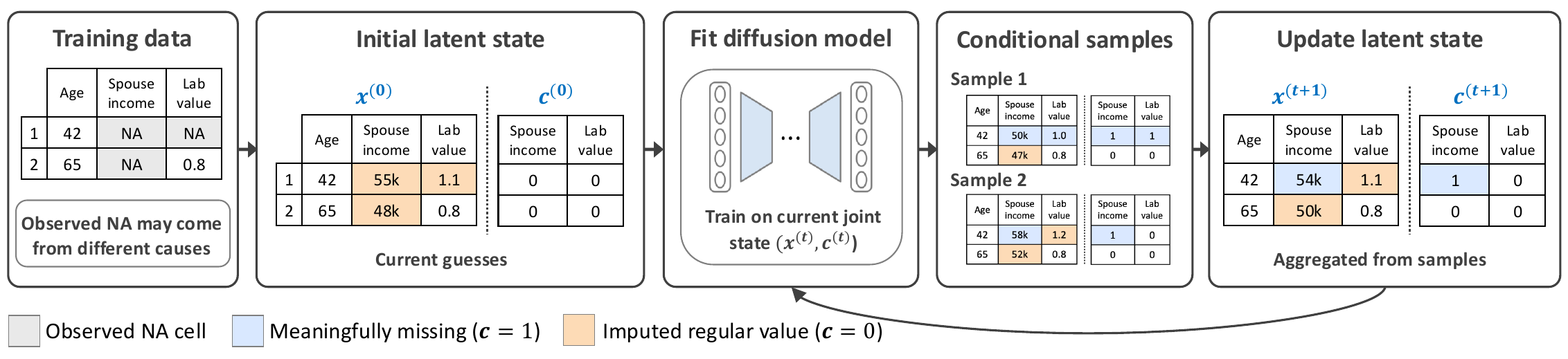}
\vspace{-0.2in}
\caption{Overview of \algname.
Starting from incomplete observations $\bx^{\mathrm{obs}}$, \algname{} alternates between diffusion-model training on the current joint state $\widehat{\bs}^{(t)} = (\widehat \bx^{(t)}, \widehat \bc^{(t)})$, and latent state update via conditional sampling and aggregation.}
\vspace{-0.1in}
\label{fig:diff_joint_overview}
\end{figure}

\subsection{\algname{}: Joint Diffusion for Selective Imputation} \label{sec:method}

To model meaningful missingness, we define the joint diffusion state
$\mathbf s = (\bx,\bc)$, where $\mathbf c$ is the meaningful-missingness (MM) mask. Based on such joint state representation, we propose \algname, an iterative framework for jointly modeling tabular values and meaningful-missingness patterns. 
As illustrated in Figure \ref{fig:diff_joint_overview}, given the observed dataset $\{\bx^\obs_i\}_{i=1}^n$, the \algname{} framework initializes the observed missing values by randomly filling the missing entries in each $\bx_i^{\obs}$ to obtain $\widehat{\bx}_i^{(0)}$, and setting
$\widehat{\bc}_i^{(0)}=\mathbf{0}$. This yields the initial joint state $\widehat{\bs}_i^{(0)}=(\widehat{\bx}_i^{(0)},\widehat{\bc}_i^{(0)})$ for each data point.
That is, all observed missing entries are initially treated as observation-induced missingness. 
Starting from $\{\widehat{\bs}_i^{(0)}\}_{i=1}^n$, \algname\ alternates between the following two steps at each iteration $t$: $(i)$ \textit{Model Update:} train a diffusion model on the given collection of joint data state $\{\widehat{\bs}_i^{(t)}\}_{i=1}^n$ to capture dependencies between data values and meaningful-missingness patterns; and $(ii)$
\textit{Latent-state Update:} draw multiple conditional samples from the current diffusion model and update the joint state to $\{\widehat{\bs}_i^{(t+1)}\}_{i=1}^n$. In the following, we describe the latent-state update and present the overall procedure in Algorithm~\ref{alg:joint_diffusion_overall}. Detailed algorithms for each step are provided in Appendix~\ref{app:alg}.


\vspace{-0.1in}
\paragraph{Latent-state Update.} At each iteration $t$, conditioning on each observed training data $\bx^{\mathrm{obs}}$, we use the current updated joint-state diffusion model (parameterized by $\boldsymbol \theta^{(t)}$) to draw $K$ samples
$\{(\bx^{(t,k)},\bc^{(t,k)})\}_{k=1}^K$. 
Here $K$ is a pre-specified sample size.  
We then aggregate these $K$ conditional samples to update both the imputed values and the meaningful-missingness mask for each data point.



First, for each missing $j$-th entry in $\bx^\obs$, we update its imputed value to $\widehat{x}_{j}^{(t+1)}$ as
\begin{equation}
\label{eq:aggregate_data}
\widehat{x}_{j}^{(t+1)}
=
\begin{cases}
\operatorname{mode}\bigl(\{x_{j}^{(t,1)},\ldots,x_{j}^{(t,K)}\}\bigr),
& \text{if feature } j \text{ is discrete},\\[6pt]
\frac{1}{K}\sum_{k=1}^K x_{j}^{(t,k)},
& \text{if feature } j \text{ is continuous},
\end{cases}
\end{equation}
where $x_j^{(t,k)}$ denotes the $j$-th entry in the generated conditional sample $\bx^{(t,k)}$.

Then, for each $j$ such that the $j$-th entry of $\bx^\obs$ is missing, to quantify how uncertain the model is about this missing entry, we define the following uncertainty score:
\begin{equation}
\label{eq:uncertainty}
u_{j}^{(t+1)}
=
\begin{cases}
-\displaystyle\sum_{v\in\mathcal{X}_j}
\hat{p}_{j}^{(t)}(v)\log \hat{p}_{j}^{(t)}(v),
& \text{if feature } j \text{ is discrete},\\[12pt]
\sqrt{\frac{1}{K}\sum_{k=1}^K
(x_{j}^{(t,k)} -\frac{1}{K}\sum_{\ell=1}^K x_{j}^{(t,\ell)})^2
},
& \text{if feature } j \text{ is continuous},
\end{cases}
\end{equation}
where
$
\hat{p}_{j}(v)=\frac{1}{K}\sum_{k=1}^K \mathbbm{1}(x_{j}^{(t,k)}=v)
$
is the empirical probability mass function of these $K$ samples. 
That is, the uncertainty score is defined as empirical entropy for discrete variables and empirical standard deviation for continuous ones. In both cases, a larger uncertainty score indicates greater posterior uncertainty about the missing entry, \ie, the generated conditional samples are more diverse.


Finally, we compute the imputation \eqref{eq:aggregate_data} and the uncertainty score \eqref{eq:uncertainty} for missing entries in each data point $\bx_i^\obs$, $i=1,\ldots,n$. That is, let $\mathcal{M}_j \subset \{1,\ldots,n\}$ denote the set of samples whose $j$-th observed entry is $\na$, then the above calculation yields the uncertainty score
$\{u_{i,j}^{(t+1)}:i\in\mathcal{M}_j\}$ for the $j$-th variable. We then apply \algnamekmean{} to separate these entries with {\it high} uncertainty from those with {\it low} uncertainty. Specifically, we denote the two resulting clusters by $\mathcal{C}_{0,j}^{(t+1)}$ and $\mathcal{C}_{1,j}^{(t+1)}$. Without loss of generality, we assume $\mathcal{C}_{1,j}^{(t+1)}$ is the cluster with higher average uncertainty values, thus treated as the meaningful-missingness cluster.
In parallel, the sampled masks provide a direct MM signal through the majority vote that equals
$\operatorname{mode}\bigl(\{ c_{i,j}^{(t,k)}\}_{k=1}^K\bigr).$
We combine these two signals conservatively to update the second part of the joint state:
\begin{equation}
\label{eq:mm_update}
\widehat c_{i,j}^{(t+1)}
=
\mathbbm{1}\{i\in\mathcal C_{1,j}^{(t+1)}\}
\lor
\operatorname{mode}\bigl(\{c_{i,j}^{(t,k)}\}_{k=1}^K\bigr).
\end{equation}
Here $\lor$ denotes the logical OR operator.
Thus, an entry is classified as meaningfully missing if either the sampled masks or the uncertainty pattern supports the MM interpretation.

The overall procedure can be viewed as an iterative scheme: the model-update step refits the diffusion model using the current joint-state estimates, while the latent-state-update step refines the MM indicators and imputed values using the learned model. Ablation studies in Appendix~\ref{app:ablation} show that the full proposed algorithm outperforms both related baseline method and variants with individual components removed, demonstrating that iterative refinement, joint-state characterization, and the uncertainty-based aggregation rule are all essential and contribute to its strong performance.

\vspace{-0.1in}
\begin{algorithm}[t!]
\caption{\algname{} for Selective Imputation}
\label{alg:joint_diffusion_overall}
\begin{algorithmic}[1]
\REQUIRE Observed data table $X^{\mathrm{obs}}=\{\bx_i^{\text{obs}}\}_{i\in [n]}$, 
observation mask $\Omega=\{\boldsymbol{\omega}_i\}_{i\in [n]}$,
MM search column set $\mathcal{T}$, number of samples $K$, number of iterations $t_{\max}$.
\ENSURE Final completed data $\widehat{X}$, meaningful-missingness mask $\widehat{C}$, and trained model parameters $\boldsymbol \theta$.

\STATE \textbf{Initialize:} randomly fill the missing entries in $X^{\mathrm{obs}}$ to obtain $\widehat{X}^{(0)}$.
\STATE Set the initial meaningful-missingness mask $\widehat{C}^{(0)} \gets \mathbf{0}_{n\times d}$.
\STATE Form the initial joint representation $\widehat{S}^{(0)} \gets (\widehat{X}^{(0)}, \widehat{C}^{(0)})$.

\FOR{$t=0,1,2,\dots, t_{\max}$}
    \STATE $\boldsymbol \theta^{(t)} \gets \textsc{Model--Update}(\widehat{S}^{(t)})$ \hfill using Algorithm~\ref{alg:mstep_modular}
    \STATE $\mathcal{S}^{(t)} \gets \textsc{Conditional--Sample}(\boldsymbol \theta^{(t)}, X^{\mathrm{obs}}, \Omega, K)$ \hfill using Algorithm~\ref{alg:estep_modular}
    \STATE $(\widehat{X}^{(t+1)}, \widehat{C}^{(t+1)}) \gets \textsc{Aggregate}(\mathcal{S}^{(t)}, \Omega, \mathcal{T})$ \hfill using Algorithm~\ref{alg:aggregate_modular}
    \STATE Form $\widehat{S}^{(t+1)} \gets (\widehat{X}^{(t+1)}, \widehat{C}^{(t+1)})$.
\ENDFOR
\STATE $\widehat X\gets \widehat{X}^{(t+1)}\odot(1-\widehat{C}^{(t+1)})+  \na\odot\widehat{C}^{(t+1)}$.
\RETURN $(\widehat{X}, \widehat{C}^{(t+1)}, \boldsymbol \theta^{(t)})$.
\end{algorithmic}
\end{algorithm}

\subsection{Identifiability of Meaningful Missingness}
\label{sec:identifiability_theory}



In this section, we provide theoretical insights into the identifiability of meaningful missingness when it is mixed with observation-induced missingness in the observed record. The following Proposition formalizes two representative regimes under which meaningful missingness can be identified. The proof of Proposition~\ref{prop:mar+known_mcar} is provided in Appendix~\ref{app:proof1}. 
Intuitively, the following proposition shows that meaningful missingness is identifiable when the observation-induced missingness channel is either known or can be recovered from the table's structural information.
The two regimes are chosen and analyzed for technical simplicity, and the same principle should extend to more general settings where the observation-induced missingness is itself identifiable or can be separated from the MM mechanism through auxiliary structure, validation information, or certain parametric conditions. Our numerical experiments in Section~\ref{sec:numerical} further show that the proposed algorithm remains fairly robust and maintains stable performance across various observation-induced missingness mechanisms.



\begin{proposition}[Identifiability]\label{prop:mar+known_mcar}
Assume that the meaningful missing mechanism satisfies $\Pr(\mathbf{c}_j=1\mid \mathbf{x}_{-j}^{\true})
= \phi_{\theta}(\mathbf{x}_{-j}^{\true};j)$ with unknown parameter $\theta$, and the function $\phi_\theta$ satisfies that $\phi_{\theta}\neq\phi_{\theta'}$ for any $\theta\neq \theta'$. We consider the following two scenarios. 
\begin{enumerate}[leftmargin=2em, itemsep=2pt, topsep=0pt, parsep=0pt, partopsep=0pt]
    \item[($i$)] Assume the observation-induced missingness is MCAR with known probabilities $p_j = \Pro(r_j=1)<1$ for all $j$. Then the meaningful missing is identifiable from the observed data $\mathbf{x}^{\mathrm{obs}}$, i.e., $\theta$ can be uniquely determined.
    \item[($ii$)] Assume that there exists a known subset $S\subseteq\{1,\ldots,d\}$, $S^c\neq\emptyset$, such that $\Pro(c_j=1)=0$ for $j\notin S$. That is, meaningful missingness will only occur in the columns in $S$. Assume further that the observation-induced missingness is MCAR with a common but unknown probability $p_j=p<1$, $j=1,\ldots,d$. Then the meaningful missing is identifiable from the observed data $\mathbf{x}^{\mathrm{obs}}$, i.e., $\theta$ can be uniquely determined. 
\end{enumerate}
\end{proposition}


\begin{remark}\label{remark:linear_regression}
We comment that Proposition~\ref{prop:mar+known_mcar} holds for general meaningfully-missing mechanisms $\phi_\theta$. One example that satisfies the assumption is the logistic scheme $\phi_{\theta}(\bx;j) = \sigma(\mathbf{\alpha}_j^\top \bx + \mathbf{\beta}_j)$, with parameters $\mathbf{\alpha}_j \in \mathbb{R}^d$ and $\beta_j \in \mathbb{R}$ identifiable under assumptions in Proposition~\ref{prop:mar+known_mcar}. We provide more discussions under this special case in the Appendix~\ref{app:proof1}.       
\end{remark}

\begin{remark}\label{remark:general}
Proposition~\ref{prop:mar+known_mcar} characterizes what is theoretically identifiable from the observed-data distribution, while practical recovery depends on the empirical effectiveness of the algorithm used.
The uncertainty score is our key design for separating MM entries from observation-induced missing entries in practice. It relies on a heuristic uncertainty-gap assumption: observation-induced missing entries correspond to regular values and tend to have concentrated conditional predictive distributions, whereas MM entries may not be well explained by any single regular state and thus tend to have larger uncertainty scores. We show in Appendix~\ref{app:theory-discuss} that, under certain conditions, the first-step iteration can guarantee that those MM entries exhibit higher uncertainty scores in expectation than non-MM entries, thus can be separated. Moreover, as the diffusion model better approximates the joint distribution, the $K$ conditional samples $\{\bc^{(t,k)}\}_{k=1}^K$ become more informative of the true MM indicators, enabling the logical OR update in Eq.~\eqref{eq:mm_update} to more effectively complement the uncertainty-gap-based update.



\end{remark}

\section{Numerical Experiments}\label{sec:numerical}

\paragraph{Datasets and Numerical Setup.}

We evaluate \algname\ on two mixed-type tabular datasets: a synthetic
Bayesian-network dataset \cite{ouyang2023missdiff} and a real-world dataset constructed from
MIMIC-IV-ED \cite{johnson2021mimic}. Table~\ref{tab:datasets_summary} in Appendix~\ref{app:data_mask_construction} summarizes the main statistics of the two datasets. 
In both datasets, we generate the ground-truth meaningful-missingness via a pre-specified mechanism, which allows us to evaluate whether a method can recover meaningful missingness from observed data. 
Specifically, in the Bayesian-network dataset, meaningful missingness is induced by the
synthetic data-generating process and can occur in both continuous and discrete variables.
In the MIMIC-IV-ED data, ground-truth meaningful-missingness labels are not directly available, so we introduce synthetic meaningful missingness through clinically motivated rules on selected discrete target variables. The detailed Bayesian-network construction and MIMIC-IV-ED feature construction are provided in Appendix~\ref{app:bn_details} and
Appendix~\ref{app:real_data_details}, respectively.
We then introduce an additional observation-induced missing layer.  We evaluate our method under missing completely at random (MCAR), missing at random (MAR), or missing not at random (MNAR) mechanisms. 
For the Bayesian-network dataset, we report results under all three ordinary-missingness mechanisms. 
For the MIMIC-IV-ED dataset, we use MCAR ordinary missingness with different masking ratios. 
The formal definitions of these observation-induced missingness mechanisms and their specific implementations are given in Appendix~\ref{app:ordinary_mechanism}.

\vspace{-0.1in}
\paragraph{Evaluation Criteria.}

We evaluate \algname\ along two dimensions: meaningful-missingness (MM)
identification and observation-induced-missing-value imputation. 
For MM identification, we report precision and recall over
the missing entries, $\mathrm{Precision}
=
\frac{\mathrm{TP}}{\mathrm{TP}+\mathrm{FP}}$, $\mathrm{Recall}
=
\frac{\mathrm{TP}}{\mathrm{TP}+\mathrm{FN}}$, 
where TP, FP, and FN are true positives, false positives, and false negatives, computed by treating meaningful missingness as the positive class.
For discrete variables, we additionally report token-level recovery accuracy, which is defined as $\mathrm{ACC}
=
\frac{1}{|\mathcal{M}_{\mathrm{disc}}|}
\sum_{(i,j)\in\mathcal{M}_{\mathrm{disc}}}
\mathbbm{1}\{\hat x_{i,j}=x_{i,j}\}$, where $\mathcal{M}_{\mathrm{disc}}$ denotes the set of missing entries in
discrete MM candidate columns; $x_{i,j} \in \bar{\calX}_j$ and $\hat x_{i,j}\in \bar{\calX}_j$ denote the true and imputed values, respectively. 
We note that this metric is stricter than MM-label accuracy: for non-MM entries, the imputed discrete value must also match the ground truth.
For continuous variables, we report RMSE (and MAE) on ordinary-missing entries in continuous columns that are not allowed to contain meaningful missingness. In all tables, ``out'' denotes test-set performance.




\begin{table}[t!]
\centering
\caption{Method comparison on Bayesian Network data under various missing mechanisms/ratios.}
\label{tab:bnc_ratio_method_compare}
\setlength{\tabcolsep}{4.5pt}
\footnotesize
\begin{tabular}{lccccccc}
\toprule
Method & Ratio (\%) & \multicolumn{2}{c}{MCAR} & \multicolumn{2}{c}{MAR} & \multicolumn{2}{c}{MNAR} \\
\cmidrule(lr){3-4}\cmidrule(lr){5-6}\cmidrule(lr){7-8}
& & ACC$_{\text{out}} \uparrow$ & RMSE$_{\text{out}} \downarrow$
& ACC$_{\text{out}} \uparrow$ & RMSE$_{\text{out}} \downarrow$
& ACC$_{\text{out}} \uparrow$ & RMSE$_{\text{out}} \downarrow$ \\
\midrule
\algname{}   & 10 & \textbf{78.34\%} & 5.09 & \textbf{75.32\%} & 5.15 & \textbf{72.51\%} & 5.16 \\
CMAE         & 10 & 35.21\% & \textbf{4.52} & 24.67\% & \textbf{4.98} & 30.43\% & \textbf{5.05} \\
DiffPuter    & 10 & 31.74\% & 5.07 & 23.54\% & 5.26 & 26.90\% & 5.29 \\
missForest   & 10 & 30.40\% & 4.76 & 19.12\% & 5.17 & 23.26\% & 5.16 \\
Mean/Mode    & 10 & 30.40\% & 4.85 & 19.12\% & 5.04 & 23.26\% & 5.08 \\
\midrule
\algname{}   & 20 & \textbf{69.35\%} & 5.08 & \textbf{66.16\%} & 5.24 & \textbf{65.85\%} & \textbf{4.88} \\
CMAE         & 20 & 47.18\% & \textbf{4.74} & 34.25\% & \textbf{5.03} & 42.14\% & 5.07 \\
DiffPuter    & 20 & 44.61\% & 5.20 & 34.79\% & 5.32 & 37.48\% & 5.36 \\
missForest   & 20 & 42.40\% & 5.01 & 30.98\% & 5.16 & 34.44\% & 5.14 \\
Mean/Mode    & 20 & 42.40\% & 4.96 & 30.98\% & 5.07 & 34.44\% & 5.07 \\
\midrule
\algname{}   & 30 & \textbf{63.65\%} & 5.06 & \textbf{57.06\%} & 5.33 & \textbf{67.91\%} & 5.32 \\
CMAE         & 30 & 52.26\% & \textbf{4.88} & 40.83\% & 5.06 & 45.29\% & 5.08 \\
DiffPuter    & 30 & 51.89\% & 5.28 & 43.48\% & 5.24 & 45.04\% & 5.27 \\
missForest   & 30 & 49.23\% & 5.04 & 41.19\% & 6.40 & 42.40\% & 5.20 \\
Mean/Mode    & 30 & 49.23\% & 5.02 & 41.19\% & \textbf{5.05} & 42.40\% & \textbf{5.05} \\
\midrule
\algname{}   & 40 & \textbf{65.92\%} & \textbf{4.99} & \textbf{67.99\%} & 5.42 & \textbf{60.38\%} & 5.21 \\
CMAE         & 40 & 54.09\% & 5.44 & 45.15\% & 5.06 & 50.30\% & 5.05 \\
DiffPuter    & 40 & 56.18\% & 5.20 & 45.27\% & 5.27 & 49.47\% & 5.22 \\
missForest   & 40 & 53.83\% & 5.21 & 42.08\% & 5.13 & 48.26\% & 5.12 \\
Mean/Mode    & 40 & 53.83\% & 5.00 & 42.08\% & \textbf{5.06} & 48.26\% & \textbf{5.02} \\
\bottomrule
\end{tabular}
\vspace{-0.1in}
\end{table}

\vspace{-0.1in}
\paragraph{Baselines.}
We compare \algname\ against representative baselines from four families: simple statistical imputation, classical iterative imputation, masked autoencoding, and diffusion-based generative imputation. Mean/Mode imputes each column independently using the empirical mean for continuous variables and the empirical mode for categorical variables. We use missForest~\cite{stekhoven2012missforest} as a strong tree-based iterative baseline for mixed-type tabular data, providing a computationally practical classical alternative on the large-scale MIMIC-IV-ED dataset. As modern deep-learning baselines, we include CACTI~\cite{gorla2025cacti}, a recent strong masked-autoencoding method reported to improve over several prior autoencoding baselines, and  DiffPuter~\cite{zhang2025diffputer}, a recent diffusion-based imputation method reported to outperform several prior diffusion baselines. On the Bayesian Network dataset, whose column names carry no semantic information, we use the non-embedding variant of CACTI, denoted CMAE. These baselines are designed to impute missing values rather than to identify meaningful missingness; therefore, their MM precision and recall are not measurable and we only report their imputation error. Additional implementation details are provided in Appendix~\ref{app:exp-details-baselines}.

\vspace{-0.1in}
\paragraph{Results on the Synthetic Dataset.}

\begin{wraptable}{r}{0.4\linewidth}
\centering
\vspace{-0.15in}
\caption{Evaluation of \algname{} on the Bayesian Network dataset under ordinary MCAR with five random seeds.}
\label{tab:bayesian_missing_summary_pr_mcar}
\scriptsize
\begin{tabular}{ccc}
\toprule
Ratio & Recall$_{\text{out}} \uparrow$ & Precision$_{\text{out}} \uparrow$\\
\midrule
10 & ${92.02\% \pm 1.34\%}$ & ${73.45\% \pm 0.34\%}$  \\
\midrule
20 & ${84.19\% \pm 2.86\%}$ & ${59.33\% \pm 1.19\%}$ \\
\midrule
 30 & ${76.55\% \pm 2.55\%}$ & ${50.80\% \pm 2.50\%}$  \\
\midrule
40 & ${76.99\% \pm 3.48\%}$ & ${43.54\% \pm 2.26\%}$ \\
\bottomrule
\end{tabular}
\vspace{-0.1in}
\end{wraptable}
Table~\ref{tab:bnc_ratio_method_compare} reports out-of-sample results on the Bayesian-network synthetic dataset under three different types of observation-induced missing: MCAR, MAR, and MNAR. It is worthwhile noting that \algname\ is the only method that explicitly identifies meaningful missingness. 
In terms of discrete token-level recovery, \algname\ achieves the best accuracy across all ratios, substantially outperforming
standard imputation baselines. This shows that modeling $\na$ as a
semantic state improves recovery of the final entry state. Moreover, \algname\ maintains stable MM recovery performance, indicating that the uncertainty-based update is not tied to a specific
missing pattern. For continuous imputation, \algname\ is competitive
but not always the best point imputer, since its objective jointly balances
ordinary-value recovery and MM identification. 
Furthermore, Table~\ref{tab:bayesian_missing_summary_pr_mcar} presents the MM precision and recall of \algname{} and it can be seen that \algname{} achieves consistently high MM recall
($92.02\%$ to $76.99\%$) as the observation-induced missing ratio increases from
$10\%$ to $40\%$. Precision decreases at higher missing ratios, reflecting the increasing difficulty of distinguishing semantic absence from randomly masked entries. Additional precision and recall results under other missing mechanisms and datasets are provided in Appendix~\ref{app:add-results}.


We further evaluate the effectiveness of \algname{} via two downstream multi-class classification tasks for target variables $\texttt{D2}$ and $\texttt{D3}$. The results are reported in Table~\ref{tab:bayesian_network_continuous_d2_d3} using Macro-F1, Weighted-F1, ROC-AUC, and accuracy as classification metrics. More details can be found in Appendix~\ref{app:imple-detail}. It can be seen that the proposed method achieves the best performance among all baseline methods, indicating that the selective imputation can significantly improve the downstream task performance when $\na$ is explicitly considered as a meaningfully missing state.

\begin{table}[!t]
\centering
\caption{Downstream performance on Bayesian Network for variables $\texttt{D2}$ and $\texttt{D3}$.}
\label{tab:bayesian_network_continuous_d2_d3}
\footnotesize
\setlength{\tabcolsep}{3.5pt}
\begin{adjustbox}{max width=\textwidth}
\begin{tabular}{lccccccccc}
\toprule
\multirow{2}{*}{Method} 
& \multirow{2}{*}{Ratio} 
& \multicolumn{4}{c}{$\texttt{D2}$} 
& \multicolumn{4}{c}{$\texttt{D3}$} \\
\cmidrule(lr){3-6} \cmidrule(lr){7-10}
& 
& Macro-F1 & ROC-AUC & Weighted-F1 & Acc.
& Macro-F1 & ROC-AUC & Weighted-F1 & Acc. \\
\midrule
\algname & 10 
& \textbf{75.54} & \textbf{94.81} & \textbf{71.42} & \textbf{89.38}
& \textbf{79.16} & \textbf{86.01} & \textbf{77.48} & \textbf{84.72} \\
Mean/Mode & 10 
& 46.50 & 73.38 & 40.21 & 72.65
& 41.69 & 52.92 & 36.32 & 62.23 \\
missForest & 10 
& 45.70 & 72.12 & 39.24 & 72.39
& 41.74 & 52.94 & 36.38 & 62.27 \\
CMAE & 10 
& 40.82 & 70.39 & 32.02 & 70.10
& 38.85 & 56.64 & 32.54 & 61.42 \\
DiffPuter & 10 
& 45.97 & 74.01 & 39.58 & 72.52
& 41.44 & 55.33 & 36.03 & 62.10 \\
\midrule
\algname & 20 
& \textbf{75.44} & \textbf{95.26} & \textbf{71.23} & \textbf{89.72}
& \textbf{79.16} & \textbf{86.25} & \textbf{77.48} & \textbf{84.72} \\
Mean/Mode & 20 
& 40.03 & 69.54 & 32.47 & 71.25
& 41.74 & 52.83 & 36.38 & 62.27 \\
missForest & 20 
& 39.59 & 68.44 & 31.97 & 70.83
& 41.74 & 52.88 & 36.38 & 62.25 \\
CMAE & 20 
& 42.10 & 72.79 & 33.48 & 70.72
& 40.88 & 54.82 & 35.12 & 62.25 \\
DiffPuter & 20 
& 46.38 & 73.92 & 40.07 & 72.62
& 41.45 & 53.98 & 36.04 & 62.13 \\
\midrule
\algname & 30 
& \textbf{76.55} & \textbf{95.49} & \textbf{72.57} & \textbf{90.00}
& \textbf{78.89} & \textbf{85.91} & \textbf{77.17} & \textbf{84.65} \\
Mean/Mode & 30 
& 37.12 & 68.25 & 28.97 & 70.75
& 41.21 & 53.21 & 35.76 & 62.00 \\
missForest & 30 
& 36.54 & 66.04 & 28.36 & 69.90
& 40.97 & 53.23 & 35.49 & 61.88 \\
CMAE & 30 
& 45.25 & 72.42 & 37.12 & 71.85
& 41.66 & 56.51 & 36.26 & 62.23 \\
DiffPuter & 30 
& 46.47 & 74.72 & 40.17 & 72.65
& 41.59 & 53.68 & 36.21 & 62.18 \\
\midrule
\algname & 40 
& \textbf{61.11} & \textbf{94.77} & \textbf{53.94} & \textbf{86.15}
& \textbf{75.13} & \textbf{86.42} & \textbf{72.73} & \textbf{83.15} \\
Mean/Mode & 40 
& 35.11 & 70.68 & 26.49 & 70.89
& 35.27 & 53.15 & 28.83 & 59.69 \\
missForest & 40 
& 33.97 & 68.49 & 25.22 & 69.94
& 34.76 & 52.96 & 28.24 & 59.42 \\
CMAE & 40 
& 38.85 & 69.95 & 29.62 & 69.08
& 39.83 & 63.21 & 33.43 & 62.23 \\
DiffPuter & 40 
& 43.23 & 74.30 & 36.11 & 72.18
& 41.62 & 53.33 & 36.23 & 62.23 \\
\bottomrule
\end{tabular}
\end{adjustbox}
\end{table}

\paragraph{Results on the MIMIC-IV-ED Dataset.}

\begin{table}[t!]
\centering
\caption{Evaluation on MIMIC-IV-ED under MCAR.}
\label{tab:mimic4ed_missing_summary_main}
\setlength{\tabcolsep}{3.2pt}
\small
\begin{tabular}{lcccc}
\toprule
Method & Ratio & MAE$_{\text{out}} \downarrow$ & RMSE$_{\text{out}} \downarrow$ & Acc$_{\text{out}} \uparrow$ \\
\midrule
\algname    & 10 & $\underline{4.59 \pm 0.03}$ & $8.31 \pm 0.07$ & $\mathbf{72.00\% \pm 0.50\%}$ \\
DiffPuter   & 10 & $4.61 \pm 0.04$ & $\underline{8.30 \pm 0.08}$ & $61.54\% \pm 0.78\%$ \\
CACTI       & 10 & $\mathbf{4.55 \pm 0.03}$ & $\mathbf{8.24 \pm 0.07}$ & $\underline{63.85\% \pm 0.61\%}$ \\
Mean/Mode   & 10 & $6.32 \pm 0.00$ & $11.39 \pm 0.00$ & $57.98\% \pm 0.00\%$ \\
missForest  & 10 & $4.69 \pm 0.02$ & $8.60 \pm 0.03$ & $62.29\% \pm 0.12\%$ \\
\midrule
\algname    & 20 & $\underline{4.78 \pm 0.03}$ & $\underline{8.68 \pm 0.04}$ & $\mathbf{72.46\% \pm 0.20\%}$ \\
DiffPuter   & 20 & $4.80 \pm 0.04$ & $8.68 \pm 0.07$ & $69.00\% \pm 0.71\%$ \\
CACTI       & 20 & $\mathbf{4.73 \pm 0.03}$ & $\mathbf{8.60 \pm 0.06}$ & $\underline{70.48\% \pm 0.59\%}$ \\
Mean/Mode   & 20 & $6.31 \pm 0.00$ & $11.37 \pm 0.00$ & $65.32\% \pm 0.00\%$ \\
missForest  & 20 & $5.02 \pm 0.02$ & $9.23 \pm 0.03$ & $69.63\% \pm 0.14\%$ \\
\midrule
\algname    & 30 & $\underline{4.99 \pm 0.03}$ & $\underline{9.05 \pm 0.05}$ & $71.84\% \pm 0.66\%$ \\
DiffPuter   & 30 & $5.01 \pm 0.03$ & $9.13 \pm 0.07$ & $71.39\% \pm 0.84\%$ \\
CACTI       & 30 & $\mathbf{4.94 \pm 0.03}$ & $\mathbf{9.04 \pm 0.06}$ & $\mathbf{72.71\% \pm 0.57\%}$ \\
Mean/Mode   & 30 & $6.31 \pm 0.00$ & $11.38 \pm 0.00$ & $68.26\% \pm 0.00\%$ \\
missForest  & 30 & $5.36 \pm 0.02$ & $9.86 \pm 0.04$ & $\underline{72.07\% \pm 0.15\%}$ \\
\midrule
\algname    & 40 & $\underline{5.20 \pm 0.03}$ & $\underline{9.39 \pm 0.05}$ & $69.54\% \pm 0.52\%$ \\
DiffPuter   & 40 & $5.27 \pm 0.03$ & $9.50 \pm 0.06$ & $73.11\% \pm 0.89\%$ \\
CACTI       & 40 & $\mathbf{5.05 \pm 0.03}$ & $\mathbf{9.25 \pm 0.06}$ & $\underline{73.52\% \pm 0.64\%}$ \\
Mean/Mode   & 40 & $6.32 \pm 0.00$ & $11.39 \pm 0.00$ & $69.77\% \pm 0.00\%$ \\
missForest  & 40 & $5.70 \pm 0.02$ & $10.44 \pm 0.04$ & $\mathbf{73.92\% \pm 0.17\%}$ \\
\bottomrule
\end{tabular}
\end{table}

Table~\ref{tab:mimic4ed_missing_summary_main} presents the results on the MIMIC-IV-ED dataset under MCAR-type observation-induced missing with varying missing ratios. \algname{} still achieves the best accuracy in terms of discrete token-level recovery for relatively small missing ratios, and it is the only method that can explicitly identify meaningful missingness. For continuous imputation, even though Diff-Joint does not always have the best performance, it remains competitive and stays as the second-best result in most cases, with very little difference from the best-performing method. This demonstrates that \algname{} enjoys a much better tradeoff between imputation performance for continuous features and the overall accuracy for discrete features. 

Furthermore, we also evaluate downstream predictive performance in 
Table~\ref{tab:mimic4ed_four_outcomes_combined}, where each method is used as a
preprocessing step before predicting four target discrete outcome variables.
\algname{} achieves consistently stronger Macro-F1 and ROC-AUC across the
outcomes, with especially large gains on predicting $\texttt{Critical}$ and $\texttt{ICU transfer 12h}$. This suggests that preserving meaningful missingness provides useful predictive signal for clinically severe outcomes, rather than simply improving cell-level imputation.

\begin{table}[tb!]
\centering
\caption{Downstream performance on MIMIC-IV-ED across four target outcome variables.}
\label{tab:mimic4ed_four_outcomes_combined}
\scriptsize
\setlength{\tabcolsep}{2.6pt}
\renewcommand{\arraystretch}{0.86}

\vspace{2pt}

\begin{adjustbox}{max width=\textwidth}
\begin{tabular}{l c cccc cccc}
\toprule
\multirow{2}{*}{Method} 
& \multirow{2}{*}{Ratio}
& \multicolumn{4}{c}{Hospitalization}
& \multicolumn{4}{c}{Critical} \\
\cmidrule(lr){3-6}\cmidrule(lr){7-10}
& & Macro-F1 & ROC-AUC & Weighted-F1 & Acc.
& Macro-F1 & ROC-AUC & Weighted-F1 & Acc. \\
\midrule
\algname & 10
& \textbf{49.14} & \textbf{70.98} & \textbf{43.71} & 61.66
& \textbf{88.25} & \textbf{98.56} & \textbf{84.58} & \textbf{95.92} \\
Mean/Mode & 10
& 40.45 & 66.37 & 34.03 & 58.97
& 51.28 & 76.26 & 34.35 & 88.56 \\
missForest & 10
& 45.36 & 68.38 & 38.41 & 63.75
& 47.93 & 78.37 & 28.81 & 88.50 \\
DiffPuter & 10
& 45.18 & 68.01 & 38.19 & 63.62
& 48.06 & 79.12 & 29.18 & 88.44 \\
CACTI & 10
& 45.65 & 68.44 & 38.68 & \textbf{64.11}
& 48.52 & 79.60 & 29.64 & 88.72 \\
\midrule
\algname & 20
& \textbf{48.99} & \textbf{70.36} & \textbf{44.34} & 57.95
& \textbf{90.11} & \textbf{98.35} & \textbf{86.79} & \textbf{96.87} \\
Mean/Mode & 20
& 37.20 & 64.71 & 31.12 & 55.83
& 53.75 & 75.81 & 38.32 & 88.42 \\
missForest & 20
& 45.40 & 68.28 & 38.46 & 63.81
& 48.86 & 78.27 & 30.23 & 88.80 \\
DiffPuter & 20
& 45.03 & 67.94 & 38.08 & 63.54
& 48.47 & 78.91 & 30.12 & 88.52 \\
CACTI & 20
& 45.67 & 68.60 & 38.67 & \textbf{64.11}
& 49.12 & 79.40 & 30.69 & 88.78 \\
\midrule
\algname & 30
& \textbf{49.24} & \textbf{71.42} & \textbf{45.22} & 54.63
& \textbf{85.82} & \textbf{98.70} & \textbf{80.84} & \textbf{96.04} \\
Mean/Mode & 30
& 34.95 & 63.81 & 29.11 & 54.52
& 54.01 & 75.36 & 38.76 & 88.42 \\
missForest & 30
& 45.58 & 68.69 & 38.58 & \textbf{64.19}
& 48.02 & 78.18 & 28.88 & 88.63 \\
DiffPuter & 30
& 44.91 & 67.88 & 37.94 & 63.41
& 48.19 & 79.08 & 29.57 & 88.31 \\
CACTI & 30
& 45.50 & 68.30 & 38.53 & 63.88
& 48.72 & 79.71 & 30.10 & 88.59 \\
\midrule
\algname & 40
& 43.32 & \textbf{70.37} & \textbf{40.36} & 45.21
& \textbf{81.79} & \textbf{97.86} & \textbf{75.29} & \textbf{95.14} \\
Mean/Mode & 40
& 31.54 & 63.15 & 26.05 & 52.22
& 52.25 & 76.43 & 36.23 & 88.14 \\
missForest & 40
& 45.46 & 68.64 & 38.44 & 64.07
& 48.24 & 78.19 & 29.26 & 88.60 \\
DiffPuter & 40
& 45.12 & 68.07 & 38.16 & 63.73
& 48.56 & 79.21 & 29.88 & 88.47 \\
CACTI & 40
& \textbf{45.73} & 68.60 & 38.74 & \textbf{64.21}
& 49.08 & 79.76 & 30.54 & 88.90 \\
\bottomrule
\end{tabular}
\end{adjustbox}

\vspace{4pt}

\begin{adjustbox}{max width=\textwidth}
\begin{tabular}{l c cccc cccc}
\toprule
\multirow{2}{*}{Method} 
& \multirow{2}{*}{Ratio}
& \multicolumn{4}{c}{ICU Transfer 12h}
& \multicolumn{4}{c}{CCI CHF} \\
\cmidrule(lr){3-6}\cmidrule(lr){7-10}
& & Macro-F1 & ROC-AUC & Weighted-F1 & Acc.
& Macro-F1 & ROC-AUC & Weighted-F1 & Acc. \\
\midrule
\algname & 10
& \textbf{90.99} & \textbf{99.65} & \textbf{87.81} & \textbf{97.32}
& \textbf{47.92} & \textbf{83.44} & \textbf{33.32} & \textbf{80.60} \\
Mean/Mode & 10
& 51.66 & 76.88 & 34.71 & 88.71
& 36.64 & 73.24 & 19.13 & 79.59 \\
missForest & 10
& 46.79 & 77.55 & 26.73 & 88.46
& 42.30 & 74.19 & 26.75 & 80.20 \\
DiffPuter & 10
& 47.64 & 80.21 & 28.42 & 88.19
& 42.76 & 73.91 & 27.21 & 79.84 \\
CACTI & 10
& 48.11 & 80.77 & 28.99 & 88.48
& 43.15 & 74.40 & 27.97 & 80.23 \\
\midrule
\algname & 20
& \textbf{90.65} & \textbf{99.54} & \textbf{87.41} & \textbf{97.20}
& \textbf{47.50} & \textbf{82.62} & \textbf{32.58} & \textbf{80.13} \\
Mean/Mode & 20
& 53.25 & 74.85 & 37.31 & 88.74
& 35.53 & 73.22 & 17.66 & 79.45 \\
missForest & 20
& 46.94 & 77.59 & 27.07 & 88.32
& 42.42 & 74.00 & 27.25 & 79.72 \\
DiffPuter & 20
& 48.51 & 79.88 & 29.96 & 88.11
& 42.91 & 73.16 & 27.84 & 79.71 \\
CACTI & 20
& 49.03 & 80.34 & 30.62 & 88.35
& 43.53 & 73.60 & 28.60 & 80.05 \\
\midrule
\algname & 30
& \textbf{88.91} & \textbf{99.63} & \textbf{84.71} & \textbf{97.32}
& \textbf{48.08} & \textbf{82.69} & \textbf{33.32} & 79.32 \\
Mean/Mode & 30
& 51.72 & 76.17 & 35.17 & 88.52
& 33.61 & 71.50 & 15.06 & 79.12 \\
missForest & 30
& 46.92 & 78.63 & 27.06 & 88.31
& 42.02 & 72.15 & 26.46 & 79.99 \\
DiffPuter & 30
& 46.88 & 79.31 & 27.21 & 88.06
& 42.54 & 73.22 & 27.36 & 79.87 \\
CACTI & 30
& 47.32 & 79.84 & 27.70 & 88.38
& 43.27 & 73.85 & 28.09 & \textbf{80.26} \\
\midrule
\algname & 40
& \textbf{85.09} & \textbf{99.37} & \textbf{79.42} & \textbf{96.43}
& \textbf{46.30} & \textbf{82.44} & \textbf{30.84} & 78.92 \\
Mean/Mode & 40
& 50.38 & 76.49 & 33.31 & 88.22
& 34.44 & 73.06 & 16.20 & 79.19 \\
missForest & 40
& 48.52 & 77.94 & 29.17 & 89.35
& 41.90 & 72.09 & 26.27 & 79.93 \\
DiffPuter & 40
& 47.41 & 80.46 & 28.02 & 88.39
& 42.08 & 73.31 & 26.79 & 79.66 \\
CACTI & 40
& 48.03 & 81.03 & 28.69 & 88.72
& 42.67 & 73.84 & 27.28 & \textbf{80.12} \\
\bottomrule
\end{tabular}
\end{adjustbox}
\end{table}

\section{Conclusion}

We introduced \algname{}, an uncertainty-aware diffusion framework for selective imputation that distinguishes meaningfully missing (MM) entries from observation-induced missing entries. By jointly modeling tabular values and MM masks, \algname{} learns when to preserve $\na$ as a semantic state and when to recover a regular value. Experiments on synthetic Bayesian-network data and MIMIC-IV-ED show that the proposed method effectively identifies meaningful missingness and yields strong downstream predictive performance. One potential limitation of the current framework is that it assumes known MM candidate columns and a detectable uncertainty gap. Although our empirical results suggest that this gap is stable across various missing mechanisms, future work could relax the current identifiability conditions. Additionally, the iterative diffusion procedure is also more computationally expensive than non-iterative baselines. However, its cost can be controlled through the user-specified number of iterations, and our ablation study shows that only a small number of iterations is often sufficient. Future work will strengthen theoretical guarantees, improve scalability, and evaluate the framework on more real-world datasets.


\bibliographystyle{unsrt} 
\bibliography{reference}


\appendix

\section{Algorithm Details}\label{app:alg}

We provide the detailed algorithms for the \textsc{Model--Update} step (Algorithm~\ref{alg:mstep_modular}), \textsc{Conditional--Sample} step (Algorithm~\ref{alg:estep_modular}), and the \textsc{Aggregate} step (Algorithm~\ref{alg:aggregate_modular}) used in Algorithm~\ref{alg:joint_diffusion_overall}.

\begin{algorithm}[h!]
\caption{\textsc{Model--Update}: Model Update}
\label{alg:mstep_modular}
\begin{algorithmic}[1]
\REQUIRE Current joint representation $\widehat{S}^{(t)}=(\widehat{X}^{(t)},\widehat{C}^{(t)})$.
\ENSURE Updated diffusion model parameters $\boldsymbol \theta^{(t)}$.

\WHILE{not converged}
    \STATE Sample a joint training example $\mathbf{s} \sim p_{\widehat{S}^{(t)}}(\mathbf{s})$.
    \STATE Sample diffusion time $t_{\mathrm{diff}} \sim p(t)$.
    \STATE Sample Gaussian noise $\boldsymbol{\varepsilon}\sim\mathcal{N}(\mathbf{0},\mathbf{I})$.
    \STATE Form the noisy sample
    $
        \mathbf{s}_{t_{\mathrm{diff}}}
        =
        \mathbf{s}
        +
        \sigma(t_{\mathrm{diff}})\boldsymbol{\varepsilon}.
    $
    \STATE Compute the score-matching loss
    \[
        \ell(\boldsymbol \theta)
        =
        \left\|
        \boldsymbol{\epsilon}_{\boldsymbol \theta}\!\left(\mathbf{s}_{t_{\mathrm{diff}}},t_{\mathrm{diff}}\right)
        -
        \frac{-\boldsymbol{\varepsilon}}{\sigma(t_{\mathrm{diff}})}
        \right\|_2^2.
    \]
    \STATE Update $\boldsymbol \theta$ via Adam.
\ENDWHILE

\RETURN $\boldsymbol \theta^{(t)} \gets \boldsymbol \theta$.
\end{algorithmic}
\end{algorithm}

\begin{algorithm}[h!]
\caption{\textsc{Conditional--Sample}}
\label{alg:estep_modular}
\begin{algorithmic}[1]
\REQUIRE Diffusion model $\boldsymbol \theta^{(t)}$, observed data $X^{\mathrm{obs}}$, observation mask $\Omega$,
number of samples $K$, number of reverse steps $M$.
\ENSURE Monte Carlo sample collection $\mathcal{S}^{(t)}=\{(X_k,C_k)\}_{k=1}^K$.

\FOR{$k=1$ to $K$}
    \STATE Sample initial noise $\tilde{\bs}_{t_M}^{(k)} \sim \mathcal{N}(\mathbf{0},\sigma^2(t_M)\mathbf{I})$.
    \FOR{$i=M,M-1,\dots,1$}
        \STATE Sample Gaussian noise $\boldsymbol{\varepsilon} \sim \mathcal N(\mathbf 0,\mathbf I)$.
        \STATE Compute forward-noised observed part:
        $
        \bs_{t_{i-1}}^{\mathrm{forward},(k)}
        =
        \bs^{\mathrm{obs}}
        +
        \sigma(t_{i-1})
        \boldsymbol{\varepsilon}.
        $
        \STATE Compute reverse update using the learned diffusion model:
        \[
        \bs_{t_{i-1}}^{\mathrm{reverse},(k)}
        = \textsc{ReverseUpdate}_{\boldsymbol \theta^{(s)}}(\tilde{\bs}_{t_i}^{(k)}, t_i, t_{i-1}).
        \]
        \STATE Merge observed and missing coordinates:
        \[
        \tilde{\bs}_{t_{i-1}}^{(k)}
        = \boldsymbol{\omega}\odot \bs_{t_{i-1}}^{\mathrm{forward},(k)}
        + (1-\boldsymbol{\omega})\odot \bs_{t_{i-1}}^{\mathrm{reverse},(k)}.
        \]
    \ENDFOR
    \STATE Decode $\tilde{\bs}_{t_0}^{(k)}$ into completed sample $(X_k,C_k)$.
\ENDFOR

\RETURN $\mathcal{S}^{(t)}=\{(X_k,C_k)\}_{k=1}^K$.
\end{algorithmic}
\end{algorithm}

\begin{algorithm}[ht!]
\caption{\textsc{Aggregate}: Aggregate Monte Carlo Samples}
\label{alg:aggregate_modular}
\begin{algorithmic}[1]
\REQUIRE Observed Data $X^{\text{obs}}$, observation mask $\Omega$,
Monte Carlo sample collection $\mathcal{S}=\{(\bx_i^{(k)},\bc_i^{(k)})\}_{i\in [n],k\in [K]}$,  MM search column set $\mathcal{T}$.
\ENSURE Updated joint state $(\widehat{X}', \widehat{C}')$ 
\STATE Initialize $\widehat X' \gets X^{\text{obs}},\widehat C' \gets \mathbf{0}_{n\times d}$, $\mathcal{M}\gets\{(i,j)\mid i\in[n],j\in[d],\omega_{i,j}=0\}$.

    \FORALL{$(i,j)\in \mathcal{M}$}
    \STATE Compute $\widehat{x}_{i,j}$, $u_{i,j}$, according to 
    \eqref{eq:aggregate_data}, \eqref{eq:uncertainty}.
    \STATE Set $\widehat X_{i,j}' \gets \widehat{x}_{i,j}$.
    \ENDFOR
    \FORALL{$j\in\mathcal{T}$}
    \STATE Let $\mathcal{M}_j \gets \{i:(i,j)\in\mathcal{M}\}$.
    
    \STATE Run \algnamekmean\ on $\{u_{i,j}: i\in\mathcal{M}_j\}$ to obtain two clusters.
    \STATE Identify the high-uncertainty cluster $\mathcal{C}_{1,j}$.
    \FORALL{$i\in\mathcal{M}_j$}
    \STATE Set $\widehat C_{i,j}' \gets \widehat{c}_{i,j}$ according to \eqref{eq:mm_update}.
    \ENDFOR
    \ENDFOR

\RETURN $(\widehat X',\widehat C')$.
\end{algorithmic}
\end{algorithm}

\section{Theoretical Properties}
We provide several basic theoretical properties of the proposed framework.

\subsection{Proof of Proposition~\ref{prop:mar+known_mcar}}\label{app:proof1}

\begin{proof}
Case 1. For each coordinate $j$, the observation process defines a known channel $K_j$ from $\bx^\true_j$ to $\bx_j^{\obs}$: 
\[ 
K_j(\bx_j^{\obs} \mid \bx^\true_j)= \begin{cases} 1, & \bx^\true_j=\na,\ \bx_j^{\obs}=\na,\\ 1-p_j, & \bx^\true_j \in\mathcal X_j,\ \bx_j^{\obs}=\bx^\true_j,\\ p_j, & \bx^\true_j \in\mathcal X_j,\ \bx_j^{\obs}=\na,\\ 0, & \text{otherwise}. \end{cases} 
\] 
Since $p_j<1$, the channel $K_j$ is injective. Indeed, for any
$a\in\mathcal X_j$,
\[
\Pro(\bx_j^{\obs}=a)=(1-p_j)\Pro(\bx_j^\true=a),
\]
and hence
\[
\Pro(\bx_j^\true=a)=\frac{\Pro(\bx_j^{\obs}=a)}{1-p_j}.
\]
Moreover,
\[
\Pro(\bx_j^{\obs}=\na)
=
\Pro(\bx_j^\true=\na)
+
p_j\Pro(\bx_j^\true\in\mathcal X_j) = \Pro(\bx_j^\true=\na)
+
p_j(1-\Pro(\bx_j^\true=\na)),
\]
so $\Pro(\bx_j^\true=\na)$ is also uniquely determined by the distribution
of $\bx_j^{\obs}$ and the known value of $p_j$.

For the full vector, the observation channel is
\[
K=K_1\otimes\cdots\otimes K_d.
\]
Since each $K_j$ is injective, the product channel $K$ is also injective.
Therefore the complete-data distribution $\Pro(\bx^\true)$ is uniquely
determined by the observed-data distribution $\Pro(\bx^{\obs})$.

Now fix a coordinate $j$. Since
\[
c_j=\mathbbm 1\{\bx_j^\true=\na\},
\]
we have, for any $z$ such that $\Pro(\bx_{-j}^\true=z)>0$,
\[
\phi_\theta(z)
=
\Pro(c_j=1\mid \bx_{-j}^{\true}=z)
=
\Pro(\bx_j^\true=\na\mid \bx_{-j}^{\true}=z).
\]
The right-hand side can be written as
\[
\Pro(\bx_j^\true=\na\mid \bx_{-j}^{\true}=z)
=
\frac{
\Pro(\bx_j^\true=\na,\ \bx_{-j}^{\true}=z)
}{
\Pro(\bx_{-j}^{\true}=z)
}.
\]
Both numerator and denominator are determined by the complete-data distribution
$\Pro(\bx^\true)$, which has already been shown to be identifiable from
$\Pro(\bx^{\obs})$. Therefore $\phi_\theta(z)$ is identifiable for every
$z$ in the support of $\bx_{-j}^{\true}$.


Therefore $\phi_\theta$ is identifiable: 
if another parameter $\theta'$ gives the same observed-data distribution,
then it must satisfy $\phi_{\theta'}(z)=\phi_\theta(z)$ for all $z$ with $\Pr(\bx_{-j}^{\true}=z)>0$. By the assumed identifiability of the parametric family $\{\phi_\theta\}$, this implies $\theta'=\theta$. Therefore the meaningful-missingness mechanism is identifiable from the observed data.

Case 2. Recall that $\omega_j$ denotes the final observation indicator. 
First consider any $j\notin S$. By assumption, $c_j=0$ almost surely.
Therefore the only possible source of observed missingness in column $j$ is
observation-induced missingness. Hence $\omega_j=1-r_j$. Since $r_j\sim\mathrm{Bernoulli}(p)$, we have $\Pro(r_j=1)=p$. The left-hand side is determined by the observed-data distribution, so $p$ is
identifiable.

Now consider $j\in S$. By definition, we have $\omega_j=0$ implies $c_j=1\ \text{or}\ r_j=1$. Since $r_j$ is MCAR and independent of $(c_j,\bx_{-j}^{\true})$,
\[
\Pro(\omega_j=0\mid \bx_{-j}^{\true}=z)
=
\Pro(c_j=1\mid \bx_{-j}^{\true}=z)
+
\Pro(c_j=0,r_j=1\mid \bx_{-j}^{\true}=z).
\]
Thus
\[
\Pro(\omega_j=0\mid \bx_{-j}^{\true}=z)
=
\phi_\theta(z)+p\{1-\phi_\theta(z)\}.
\]
Equivalently,
\[
\Pro(\omega_j=0\mid \bx_{-j}^{\true}=z)
=
p+(1-p)\phi_\theta(z).
\]
Because $p<1$ is identifiable from the non-MM columns, we obtain
\[
\phi_\theta(z)
=
\frac{\Pro(\omega_j=0\mid \bx_{-j}^{\true}=z)-p}{1-p}.
\]
The right-hand side is identifiable from the observed-data distribution.
Therefore $\phi_\theta$ is identifiable on the support of $\bx_{-j}^{\true}$.

Finally, suppose another parameter $\theta'$ gives the same observed-data
distribution. Then it gives the same function $\phi_{\theta'}(z)$ on the support
of $\bx_{-j}^{\true}$. Hence $\phi_{\theta'}(z)=\phi_\theta(z)$ for all $z$ in the support. By the identifiability assumption on the parametric
family, this implies $\theta'=\theta$. 
Therefore $\theta$ is identifiable.
\end{proof}

\subsection{More Discussions}\label{app:theory-discuss}

\paragraph{Discussion of Remark \ref{remark:linear_regression}}

According to proof under Case 2, 
we obtain
\[
\phi_\theta(z)
=\sigma(\mathbf{\alpha}_j^\top z+ \beta_j)=
\frac{\Pro(\omega_j=0\mid \bx_{-j}^{\true}=z)-p}{1-p}.
\]

In typical applications, the sample size is much larger than the
number of covariates. Moreover, under a mild assumption that for any $(\mathbf{\alpha},\beta)$,
\[
\mathbf{\alpha}^\top Z + \beta = 0 \quad \text{a.s.}
\]
implies $\mathbf{\alpha} = \mathbf{0}$ and $\beta = 0$
\footnote{This assumption can be induced by the general assumption $\phi_{\theta}\neq\phi_{\theta'}$ for any $\theta\neq \theta'$.}, the linear predictor
$\mathbf{\alpha}_j^\top z + \beta_j$ is uniquely determined, and hence
$(\mathbf{\alpha}_j, \beta_j)$ is unique.

\paragraph{Discussion of Remark~\ref{remark:general}} 

We provide a more in-depth discussion of the uncertainty separation exhibited in the first iteration. 
For a distribution \(P\) on \(\mathcal X_j\), define the population uncertainty
functional
\[
\mathcal U_j(P)
=
\begin{cases}
-\displaystyle\sum_{v\in\mathcal X_j} P(v)\log P(v),
& \text{if feature } j \text{ is discrete},\\[8pt]
\sqrt{\operatorname{Var}_{Y\sim P}(Y)},
& \text{if feature } j \text{ is continuous}.
\end{cases}
\]

\begin{proposition}[First-iteration weak separation]
\label{prop:first_iteration_weak_separation}
Fix a candidate feature \(j\). For each random sample $\bx^{\obs}:=\{x^\obs_1,\ldots,x^\obs_d\}$, let
$\mathcal F^{\obs}_{j}
=
\left(\bx^{\obs}_{-j},\boldsymbol{\omega}_{-j}\right)$
be the observed context excluding the target coordinate \(j\). Suppose that the
observation-induced missingness is MCAR, $r_{j}\sim \operatorname{Bernoulli}(\rho_j)$, and that the meaningful-missingness mechanism is MAR with respect to the
observed context:
\[
\Pro(c_{j}=1\mid \mathcal F^{\obs}_{j})
=
\pi_j(\mathcal F^{\obs}_{j}).
\]
Let $\bar\pi_j=\mathbb E\!\left[\pi_j(\mathcal F^{\obs}_{j})\right]$.
Define the population conditional distribution of the initialized first-round value by
\[
P^{(0)}_j(A\mid F)
:=
\Pro
\left(
\widehat X^{(0)}_{j}\in A
\mid
\mathcal F^{\obs}_{j}=F
\right),
\qquad A\subseteq \mathcal X_j .
\]
The corresponding first-iteration population uncertainty score for the $j$-th coordinate is
$U^{(1)}_{j}
:=
\mathcal U_j
\left(
P^{(0)}_j(\cdot\mid \mathcal F^{\obs}_{i,j})
\right)$.
Assume that \(\rho_j>0\), \(U^{(1)}_{j}\) has finite first moment and
$\operatorname{Cov}
\left(
\pi_j(\mathcal F^{\obs}_{j}),
U^{(1)}_{j}
\right)
>0$.
Then the first-iteration population uncertainty score is larger on average for
meaningfully missing entries than for observation-induced missing entries:
\[
\mathbb E
\left[
U^{(1)}_{j}
\mid
c_{j}=1
\right]
>
\mathbb E
\left[
U^{(1)}_{j}
\mid
c_{j}=0,r_{j}=1
\right].
\]
\end{proposition}

\begin{proof}
By definition,
$U^{(1)}_{j}
=
\mathcal U_j
\left(
P^{(0)}_j(\cdot\mid \mathcal F^{\obs}_{j})
\right)$,
so \(U^{(1)}_{j}\) is measurable with respect to
\(\mathcal F^{\obs}_{j}\). Since
$\Pro(c_{j}=1\mid \mathcal F^{\obs}_{j})
=
\pi_j(\mathcal F^{\obs}_{j})$,
we have
\[
\begin{aligned}
\mathbb E
\left[
U^{(1)}_{j}
\mid
c_{j}=1
\right]& =
\frac{
\mathbb E
\left[
U^{(1)}_{j}\mathbbm 1\{c_{j}=1\}
\right]
}{
\Pro(c_{j}=1)
}
=
\frac{
\mathbb E
\left[
U^{(1)}_{j}
\Pro(c_{j}=1\mid \mathcal F^{\obs}_{j})
\right]
}{
\bar\pi_j
}
=
\frac{
\mathbb E
\left[
\pi_j(\mathcal F^{\obs}_{j})U^{(1)}_{j}
\right]
}{
\bar\pi_j
}.
\end{aligned}
\]
Similarly, since \(r_{j}\) is MCAR and independent of
\((c_{j},\mathcal F^{\obs}_{j})\),
\[
\begin{aligned}
\mathbb E
\left[
U^{(1)}_{j}
\mid
c_{j}=0,r_{j}=1
\right]
=
\mathbb E
\left[
U^{(1)}_{j}
\mid
c_{j}=0
\right]
&=
\frac{
\mathbb E
\left[
U^{(1)}_{j}\mathbbm 1\{c_{j}=0\}
\right]
}{
\Pro(c_{j}=0)
}\\&=
\frac{
\mathbb E
\left[
U^{(1)}_{j}
\left(
1-\pi_j(\mathcal F^{\obs}_{j})
\right)
\right]
}{
1-\bar\pi_j
}.
\end{aligned}
\]
Therefore,
\[
\begin{aligned}
&\mathbb E
\left[
U^{(1)}_{j}
\mid
c_{j}=1
\right]
\!-\!
\mathbb E
\left[
U^{(1)}_{j}
\mid
c_{j}=0,r_{j}=1
\right] 
\!\!=\!\!
\frac{
\mathbb E
\left[
\pi_j(\mathcal F^{\obs}_{j})U^{(1)}_{j}
\right]
}{
\bar\pi_j
}
\!-\!
\frac{
\mathbb E
\left[
\left(
1-\pi_j(\mathcal F^{\obs}_{j})
\right)
U^{(1)}_{j}
\right]
}{
1-\bar\pi_j
}
\\
&=
\frac{
\mathbb E
\left[
\pi_j(\mathcal F^{\obs}_{j})U^{(1)}_{j}
\right]
-
\bar\pi_j
\mathbb E
\left[
U^{(1)}_{j}
\right]
}{
\bar\pi_j(1-\bar\pi_j)
}=
\frac{
\operatorname{Cov}
\left(
\pi_j(\mathcal F^{\obs}_{j}),
U^{(1)}_{j}
\right)
}{
\bar\pi_j(1-\bar\pi_j)
}.
\end{aligned}
\]
\end{proof}

By the results in Proposition~\ref{prop:first_iteration_weak_separation}, there exists a threshold \(\tau_j\) such that
\[
\Pro
\left(
U^{(1)}_{j}>\tau_j
\mid
c_{j}=1
\right)
>
\Pro
\left(
U^{(1)}_{j}>\tau_j
\mid
c_{j}=0,r_{j}=1
\right).
\]
For such a threshold, the high-uncertainty group can thus be distinguished from the observation-induced missing group. 

\begin{remark}[Interpretation of the covariance condition]
\label{rem:cov_condition_random_init}
The covariance condition in
Proposition~\ref{prop:first_iteration_weak_separation} is natural under the
random uniform initialization used by \algname{}. Let \(G_j\) denote the empirical initialization distribution
\(\operatorname{Unif}(\widehat{\mathcal X}_j)\) in
Appendix~\ref{app:imple-detail}. For any measurable set
\(A\subseteq\mathcal X_j\), define
\[
Q_j(A\mid F)
:=
\Pro(X_j\in A\mid c_j=0,\mathcal F_j^{\obs}=F),
\]
the context-specific conditional distribution. Since every
observed missing entry is initially filled from \(G_j\), while an observed
regular entry keeps its true value, the first-round initialized distribution
satisfies
\[
P_j^{(0)}(\cdot\mid F)
=
(1-\alpha_j(F))Q_j(\cdot\mid F)
+
\alpha_j(F)G_j,
\qquad
\alpha_j(F)=\rho_j+(1-\rho_j)\pi_j(F).
\]
Thus, the mixture weight on the random-initialization component increases with
the MM propensity \(\pi_j(F)\). The initialization distribution \(G_j\) ignores
the sample-specific context \(F\) and spreads mass over the column-level
regular-value support, whereas \(Q_j(\cdot\mid F)\) is constrained by the
observed context and is typically more concentrated when the regular value is
recoverable. Therefore, higher-MM-propensity contexts tend to induce larger
first-iteration uncertainty,
$U_j^{(1)}=\mathcal U_j(P_j^{(0)}(\cdot\mid \mathcal F_j^{\obs}))$,
giving a positive association between
\(\pi_j(\mathcal F_j^{\obs})\) and \(U_j^{(1)}\).
\end{remark}

\section{Dataset and Missingness Construction}
\label{app:data_mask_construction}

This appendix describes the synthetic and real-world datasets used in the experiments, as summarized in Table~\ref{tab:datasets_summary}.  
For each dataset, we first specify the underlying data or feature construction, then define the meaningful-missingness mechanism, and finally introduce an additional missing mask to represent the observation-induced missingness. 
The observation-induced missing mechanisms are shared across two datasets and are therefore presented first. Throughout the experiments, the set of MM candidate columns is assumed to be known, and MM identification is performed only on missing entries in these candidate columns.

\begin{table}[ht!]
\centering
\caption{Summary of the datasets used in the experiments.}
\footnotesize
\label{tab:datasets_summary}
\begin{tabular}{c|c|c|c|c|c|c}
\toprule
Dataset & \#Train & \#Test & \#Continuous & \#Discrete & \#Continuous MM  & \#Discrete MM  \\
\midrule
Bayesian Network &  14000 & 6000 & 2 & 3 & 1 & 2 \\
MIMIC-IV-ED &  353150 & 88287 & 14 & 16 & 0 & 4 \\
\bottomrule
\end{tabular}
\end{table}

\subsection{Bayesian Network Synthetic Data Construction}
\label{app:bn_details}

We construct the synthetic tabular dataset from a Bayesian network so that both the underlying data-generating process and the meaningful-missingness mechanism are fully controlled. This synthetic data contains five variables: two continuous variables $\texttt{C1}$ and $\texttt{C2}$, and three discrete variables $\texttt{D1}$, $\texttt{D2}$, and $\texttt{D3}$.
The corresponding graph structure is shown in Figure~\ref{fig:BN}.

\begin{figure}[h!]
\centering
  \includegraphics[scale=0.9]{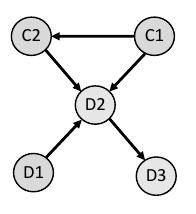}
  \caption{Bayesian network used to generate the synthetic tabular data. $\texttt{C1}$ and $\texttt{C2}$ are continuous variables, while $\texttt{D1}$, $\texttt{D2}$, and $\texttt{D3}$ are discrete variables.}
  \label{fig:BN}
\end{figure}

According to the underlying dependence in Figure~\ref{fig:BN}, we first sample $\texttt{C1} \sim \mathcal{N}(25,2)$. 
Next, conditional on $\texttt{C1}$, we generate
\[
\texttt{C2}\mid \texttt{C1} \sim \mathcal{N}(0.1\cdot \texttt{C1}+50,\,5).
\]

We then generate $\texttt{D1}\sim \mathrm{Bernoulli}(0.3)$, where $\mathrm{Bernoulli}(\xi)$ denotes the Bernoulli distribution with mean $\xi$. 
The variable $\texttt{D2}$ is generated conditionally on $\texttt{C1}$, $\texttt{C2}$, and $\texttt{D1}$ as
\[
\texttt{D2}\mid \texttt{C1},\texttt{C2},\texttt{D1}\sim
\begin{cases}
Ca(0.3,0.6,0.1), & \texttt{C1}>26,\ \texttt{C2}>55,\ \texttt{D1}=1,\\
Ca(0.2,0.3,0.5), & \texttt{C1}>26,\ \texttt{C2}\le 55,\ \texttt{D1}=1,\\
Ca(0.7,0.1,0.2), & \texttt{C1}\le 26,\ \texttt{C2}>55,\ \texttt{D1}=1,\\
\na, & \texttt{C1}\le 26,\ \texttt{C2}\le 55,\ \texttt{D1}=1,\\
Ca(0.05,0.05,0.9), & \texttt{D1}=0,
\end{cases}
\]
where $Ca(p_1,p_2,1-p_1-p_2)$ denotes a categorical distribution over three regular categories.

Finally, $\texttt{D3}$ is generated conditionally on $\texttt{D2}$ as
\[
\texttt{D3}\mid \texttt{D2}\sim
\begin{cases} 
\mathrm{Bernoulli(0.2)} & \text{D2}=0; \\
\na & \text{D2}=1; \\
\mathrm{Bernoulli(0.8)} & \text{D2}=2;\\
\na & \text{D2}=\na.
\end{cases}
\]

After generating the five variables, we further introduce meaningful missingness on $\texttt{C1}$ by recoding large values of $\texttt{C1}$ into a semantic missing state:
\[
\texttt{C1}=
\begin{cases}
\na, & \texttt{C1}>27,\\
\texttt{C1}, & \texttt{C1}\le 27.
\end{cases}
\]
Overall, the data-generating distribution for this synthetic dataset has meaningfully missing variables in both the continuous variable $\texttt{C1}$ and two discrete variables $\texttt{D2}$, $\texttt{D3}$.

\subsection{MIMIC-IV-ED Dataset}
\label{app:real_data_details}

We extract the real-world dataset from MIMIC-IV-ED \cite{johnson2021mimic} by selecting a structured feature subset and then injecting synthetic meaningful missingness on pre-specified target columns.

\paragraph{Cohort and Feature Construction.}
\label{app:mimic_data}

Our real-data dataset is constructed from MIMIC-IV-ED using a selected subset of structured ED variables \cite{johnson2021mimic}. 
We retain 30 non-object columns, including 14 numerical variables and 16 discrete variables, covering triage measurements, last-recorded vital signs, prior utilization counts, clinical scores, outcomes, chief complaint indicators, and comorbidity indicators; the full list is given in Table~\ref{tab:real_dataset_selected}.

The feature subset is chosen so that the recorded values can serve as reliable ground truth for controlled evaluation and so that the selected covariates support the construction of clinically meaningful MM mechanisms on a subset of target columns. 
Based on this subset, we later introduce synthetic meaningful missingness and additional observation-induced missingness to form the final dataset.

\begin{table}[h!]
\centering
\caption{Selected 30-column feature subset used to construct the MIMIC-IV-ED real-data dataset.}
\label{tab:real_dataset_selected}
\small
\begin{tabular}{
>{\raggedright\arraybackslash}p{0.24\linewidth}
>{\raggedright\arraybackslash}p{0.54\linewidth}
>{\raggedright\arraybackslash}p{0.12\linewidth}
}
\toprule
\textbf{Group} & \textbf{Columns} & \textbf{Type} \\
\midrule

Triage vitals
&
\texttt{triage\_temperature}, \texttt{triage\_heartrate}, \texttt{triage\_resprate}, \texttt{triage\_o2sat}, \texttt{triage\_sbp}, \texttt{triage\_dbp}, \texttt{triage\_pain}, \texttt{triage\_acuity}
& Continuous \\

Last recorded vitals
&
\texttt{ed\_temperature\_last}, \texttt{ed\_heartrate\_last}, \texttt{ed\_resprate\_last}, \texttt{ed\_o2sat\_last}, \texttt{ed\_sbp\_last}, \texttt{ed\_dbp\_last}
& Continuous \\

Demographics / utilization
&
\texttt{age}, \texttt{n\_ed\_365d}, \texttt{n\_hosp\_365d}, \texttt{n\_icu\_365d}
& Discrete \\

Clinical scores
&
\texttt{score\_CCI}, \texttt{score\_NEWS}
& Discrete \\

Outcomes
&
\texttt{outcome\_hospitalization}, \texttt{outcome\_critical}, \texttt{outcome\_icu\_transfer\_12h}
& Discrete \\

Chief complaint indicators
&
\texttt{chiefcom\_chest\_pain}, \texttt{chiefcom\_shortness\_of\_breath}, \texttt{chiefcom\_abdominal\_pain}, \texttt{chiefcom\_fever\_chills}
& Discrete \\

Comorbidity indicators
&
\texttt{cci\_CHF}, \texttt{cci\_Renal}, \texttt{eci\_Pulmonary}
& Discrete \\
\bottomrule
\end{tabular}
\end{table}

\paragraph{Meaningful-Missingness Mechanisms on MIMIC-IV-ED.}
\label{app:mimic_mm_mechanism}

After constructing the MIMIC-IV-ED cohort, we introduce synthetic meaningful missingness on a pre-specified subset of discrete variables. 
In the current dataset, the MM target columns are
\[
\texttt{outcome\_critical},\,
\texttt{outcome\_icu\_transfer\_12h},\,
\texttt{outcome\_hospitalization}, \,
\texttt{cci\_CHF}.
\]
Table~\ref{tab:mm_ratio_mimic} reports the
resulting MM rates, computed before adding the additional
observation-induced missingness mask.
\begin{table}[h]
\centering
\caption{Meaningful-missingness rates induced by the synthetic MM mechanisms on selected MIMIC-IV-ED target columns.}
\label{tab:mm_ratio_mimic}
\begin{tabular}{llc}
\toprule
Feature &  MM Rate (\%) \\
\midrule
\texttt{outcome\_critical}  & 9.06 \\
\texttt{outcome\_icu\_transfer\_12h} & 8.70 \\
\texttt{outcome\_hospitalization}  & 13.82 \\
\texttt{cci\_CHF} & 15.17 \\
\bottomrule
\end{tabular}
\end{table}
\paragraph{MM on \texttt{outcome\_critical}.}
We set \(c_{i,j}=1\) for \(j=\texttt{outcome\_critical}\) if
$\texttt{triage\_acuity}_i \le 2$
and at least one of the following conditions holds:
\[
\texttt{triage\_o2sat}_i < 95,\qquad
\texttt{triage\_resprate}_i \ge 22,\qquad
\texttt{triage\_heartrate}_i \ge 110,
\]
and set \(c_{i,j}=0\) otherwise. 
\paragraph{MM on \texttt{outcome\_icu\_transfer\_12h}.}
We set \(c_{i,j}=1\) for \(j=\texttt{outcome\_icu\_transfer\_12h}\) if
$\texttt{triage\_acuity}_i \le 2$
and at least one of the following conditions holds:
\[
\texttt{triage\_o2sat}_i < 95,\qquad
\texttt{triage\_resprate}_i \ge 22,\qquad
\texttt{n\_icu\_365d}_i \ge 1,
\]
and set \(c_{i,j}=0\) otherwise. 
\paragraph{MM on \texttt{outcome\_hospitalization}.}
For hospitalization, we define the score
\begin{align*}
s_i^{\mathrm{hosp}}
=&
\mathbf{1}\{\texttt{triage\_acuity}_i \le 2\}
+
\mathbf{1}\{\texttt{score\_NEWS}_i \ge 2\}\\
&+
\mathbf{1}\{\texttt{age}_i \ge 70\}
+
\mathbf{1}\{\texttt{n\_hosp\_365d}_i \ge 1\},
\end{align*}
and map it to a missingness probability
$p_i^{\mathrm{hosp}}
=
\mathrm{clip}\!\left(0.02 + 0.10\, s_i^{\mathrm{hosp}},\, 0,\, 0.45\right)$.
We then sample
\[
c_{i,j}\sim \mathrm{Bernoulli}\!\left(p_i^{\mathrm{hosp}}\right),
\qquad
j=\texttt{outcome\_hospitalization}.
\]

\paragraph{MM on \texttt{cci\_CHF}.}
For the CHF comorbidity indicator, we define
\begin{align*}
s_i^{\mathrm{CHF}}
=&
\mathbf{1}\{\texttt{age}_i \ge 75\}
+  \mathbf{1}\{\texttt{n\_hosp\_365d}_i \ge 2\}\\
&+\mathbf{1}\{\texttt{triage\_sbp}_i < 110\}
+\mathbf{1}\{\texttt{score\_CCI}_i \ge 4\},
\end{align*}
and map it to a probability
$p_i^{\mathrm{CHF}}
=
\mathrm{clip}\!\left(0.03 + 0.16\, s_i^{\mathrm{CHF}},\, 0,\, 0.70\right)$.
We then sample
\[
c_{i,j}\sim \mathrm{Bernoulli}\!\left(p_i^{\mathrm{CHF}}\right),
\qquad
j=\texttt{cci\_CHF}.
\]

\subsection{Observation-induced Missingness Mechanisms}
\label{app:ordinary_mechanism}

For both the synthetic and real-data datasets, we introduce an additional observation-induced ordinary missingness mask $R=(r_{i,j})$ on top of the underlying samples. Depending on the experiment, this mask is generated under MCAR, MAR, or MNAR mechanisms as defined below.

\paragraph{Missing Completely at random (MCAR).}
\label{app:mcar_mechanism}
Under the MCAR mechanism, each entry is masked independently with a fixed probability $p$, regardless of the data values. Formally, for each sample $i$ and feature $j$, the mask variable $r_{i,j}$ is generated as
\[
\Pro(r_{i,j}=1 \mid X) = p.
\]
Therefore, the resulting mask is completely random and has expected missing proportion $p$.

\paragraph{Missing at Random (MAR).}
\label{app:mar_mechanism}
Under the MAR mechanism, the missingness of some variables depends on a subset of variables that remain fully observed. Specifically, the features are first split into two disjoint sets:
\[
\mathcal{J}_{\mathrm{obs}} \cup \mathcal{J}_{\mathrm{miss}} = \{1,\dots,d\},
\qquad
\mathcal{J}_{\mathrm{obs}} \cap \mathcal{J}_{\mathrm{miss}} = \emptyset,
\]
where variables in $\mathcal{J}_{\mathrm{obs}}$ are always observed, and variables in $\mathcal{J}_{\mathrm{miss}}$ may be masked. For each $j \in \mathcal{J}_{\mathrm{miss}}$, the masking probability is defined through a logistic model:
\[
\Pro(r_{i,j}=1 \mid X)
=
\Pro(r_{i,j}=1 \mid X_{i,\mathcal{J}_{\mathrm{obs}}})
=
\sigma\!\bigl(X_{i,\mathcal{J}_{\mathrm{obs}}}^{\top} w_j + b_j\bigr),
\]
where $\sigma(x)=1/(1+e^{-x})$ is the sigmoid function, $w_j$ is a randomly generated coefficient vector, and $b_j$ is chosen so that the average masking rate is approximately $p$. Then the mask is sampled as
\[
r_{i,j} \sim \mathrm{Bernoulli}\!\left(\sigma\!\bigl(X_{i,\mathcal{J}_{\mathrm{obs}}}^{\top} w_j + b_j\bigr)\right),
\qquad j \in \mathcal{J}_{\mathrm{miss}}.
\]
Hence, the probability of missingness depends only on other observed variables, not on the masked entry itself, which is exactly the defining property of MAR.

\paragraph{Missing Not at Random (MNAR).}
\label{app:mnar_mechanism}
Under the MNAR logistic mechanism, the features are first partitioned into two sets:
\[
\mathcal{J}_{\mathrm{param}} \cup \mathcal{J}_{\mathrm{miss}} = \{1,\dots,d\},
\qquad
\mathcal{J}_{\mathrm{param}} \cap \mathcal{J}_{\mathrm{miss}} = \emptyset.
\]
The variables in $\mathcal{J}_{\mathrm{param}}$ are used as inputs to a logistic masking model, while variables in $\mathcal{J}_{\mathrm{miss}}$ are masked according to
\[
\Pro(r_{i,j}=1 \mid X)
=
\sigma\!\bigl(X_{i,\mathcal{J}_{\mathrm{param}}}^{\top} w_j + b_j\bigr),
\qquad j \in \mathcal{J}_{\mathrm{miss}}.
\]
Thus,
\[
r_{i,j} \sim \mathrm{Bernoulli}\!\left(\sigma\!\bigl(X_{i,\mathcal{J}_{\mathrm{param}}}^{\top} w_j + b_j\bigr)\right),
\qquad j \in \mathcal{J}_{\mathrm{miss}}.
\]
After this, the input variables themselves are additionally masked at random:
\[
r_{i,j} \sim \mathrm{Bernoulli}(p),
\qquad j \in \mathcal{J}_{\mathrm{param}}.
\]
Therefore, the missingness of variables in $\mathcal{J}_{\mathrm{miss}}$ depends on values from $\mathcal{J}_{\mathrm{param}}$, but those driving variables may themselves become missing. As a result, the masking probability depends on information that is not fully observed in the final dataset, so the mechanism is Missing Not At Random.

\section{Implementation Details and Additional Numerical Results}\label{app:more-results}


\subsection{Implementation Details}\label{app:imple-detail}

We implement \algname\ using the EDM diffusion framework~\citep{karras2022elucidating}.
At each outer iteration, the model is trained on the current joint state
\[
    \bs = [\bx_{\mathrm{model}}, \bc]\in\mathbb{R}^{d_x+d_c},
\]
where \(\bx_{\mathrm{model}}\) is the normalized encoded tabular vector and
\(\bc\) is the current predicted meaningful-missingness mask in the original
raw-column space. The mask coordinates are represented as scalar binary
variables and concatenated directly with the encoded data representation.

We use the same TabDDPM-style MLP denoising backbone as
DiffPuter~\citep{pmlr-v202-kotelnikov23a,zhang2025diffputer}, applied to the
joint diffusion state rather than to the encoded data vector alone. The EDM
network consists of an input projection, a sinusoidal noise embedding passed
through a two-layer time MLP and added to the projected state, three fully connected layers with SiLU activations, and a linear output layer matching the joint-state dimension. Thus, the denoiser is trained to reconstruct both the tabular coordinates and the MM-mask coordinates under the standard EDM preconditioning. Noise levels during training follow the EDM log-normal noise distribution, and sampling uses the standard EDM \(\rho\)-schedule with Heun's second-order sampler. We use the same architecture across all datasets and missingness settings.


\paragraph{Random Initialization.}
At initialization, all predicted meaningful-missingness indicators are set to
zero, i.e., \(\widehat C^{(0)}=0\), so every observed missing entry is initially
treated as observation-induced missingness. For each feature \(j\), we form the
empirical regular-value support from the observed entries,
\[
\widehat{\mathcal X}_j
=
\{x^{\obs}_{\ell,j}: \omega_{\ell,j}=1,\ \ell\in[n]\}.
\]
Then, for every missing entry \((i,j)\), we initialize
\[
    \widehat x^{(0)}_{i,j}
    =
    G_{i,j},
    \qquad
    G_{i,j}\sim \operatorname{Unif}(\widehat{\mathcal X}_j),
\]
independently across missing entries.

\paragraph{Discrete-variable Handling.}
We distinguish integer-valued and categorical discrete variables. Integer-valued
columns are encoded as scalar coordinates in the diffusion model. During decoding,
these coordinates are rounded to the nearest valid integer value. Categorical
columns are expanded into one-hot blocks before training. After sampling, each
categorical block is decoded by taking the \(\arg\max\) over the block.

\paragraph{Diff-Joint Hyperparameters.}
We use 10 iterations for the Bayesian Network and 3 iterations for MIMIC-IV-ED. Except for the number of refinement iterations, we use a fixed
hyperparameter configuration across datasets, missingness ratios, missingness
mechanisms, and random seeds. The remaining
default settings are summarized in Table~\ref{tab:diffjoint_hyperparams}.

\paragraph{Hardware and Runtime.}
All experiments were conducted on a single NVIDIA RTX 5090 GPU. 
Each combination of random seed, missingness mechanism, and missingness ratio
takes approximately 15 minutes on the Bayesian-network dataset and 2 hours on
MIMIC-IV-ED.

\paragraph{Downstream Task Setup.}
All downstream tasks are formulated as multi-class classification problems on
the augmented feature domain, where \texttt{na} may be treated as a valid
semantic class when it appears in the target variable. We evaluate downstream
utility using a train-on-generated, test-on-real classification protocol. 
For each method, dataset, missingness ratio, and random seed, let
$\widehat X_{\mathrm{tr}}$ denote the method-specific generated or completed
training table obtained from the observed training data. Let the $j$-th variable
be the target variable of the downstream classification task. The classifier is
trained with $\widehat X_{\mathrm{tr},-j}$ as features and
$\widehat X_{\mathrm{tr},j}$ as labels.

Hyperparameters are selected on the real augmented-domain training split, and
final performance is evaluated on the real augmented-domain test split. The same
classifier family, feature-processing pipeline, and hyperparameter grid are used
for all methods under the same downstream task.

When the target variable contains a meaningful-missing state, \texttt{na} is
treated as a valid semantic class rather than as an unobserved label.
Categorical \texttt{na} values in input features are encoded as ordinary
categorical states, while continuous MM candidate columns are removed from the
downstream feature set. On the Bayesian Network dataset, we evaluate two
multi-class tasks with target variables $\texttt{D2}$ and $\texttt{D3}$. On MIMIC-IV-ED, we evaluate four
discrete clinical prediction tasks:
\texttt{outcome\_hospitalization},
\texttt{outcome\_critical},
\texttt{outcome\_icu\_transfer\_12h}, and
\texttt{cci\_CHF}. We use an XGBoost classifier for all tasks and select the
model by validation Macro-F1 on the real training split.

We report Macro-F1, Weighted-F1, ROC-AUC, and accuracy. Macro-F1 is the
unweighted average of class-wise F1 scores. Weighted-F1 denotes the class-balanced weighted F1 score used by our evaluator, where the class-wise F1 score for class $\ell$ is weighted by $\frac{1-p_\ell}{L-1}$, with $L$ denoting the number of classes and $p_\ell$ denoting the class proportion. This weighting assigns larger relative weights to lower-support classes.
ROC-AUC measures the ranking quality of
predicted class probabilities and uses one-vs-rest aggregation for multi-class
tasks. Accuracy is the fraction of correctly classified test examples.

\begin{table}[t!]
\centering
\caption{Default hyperparameters for Diff-Joint.}
\label{tab:diffjoint_hyperparams}
\begin{tabular}{ll}
\toprule
Hyperparameter & Value \\
\midrule
Number of samples $K$ & \(30\) \\
Sampling steps & \(50\) \\
MLP hidden/time-embedding dimension & \(1024\) \\
Batch size & \(4096\) \\
Optimizer & Adam \\
Learning rate & \(10^{-4}\) \\
Maximum epochs per outer iteration & \(1001\) \\
Early stopping patience & \(200\) \\
\(P_{\mathrm{mean}}\), \(P_{\mathrm{std}}\) & \(-1.2,\ 1.2\) \\
\(\sigma_{\mathrm{data}}\) & \(0.5\) \\
\(\sigma_{\min}\), \(\sigma_{\max}\), \(\rho\) & \(0.002,\ 80,\ 7\) \\
Inner inpainting repeats per noise level & \(10\) \\
\bottomrule
\end{tabular}
\end{table}

\subsection{Additional Results}\label{app:add-results}

We also provide the precision and recall of the proposed \algname{} algorithm under three different missing mechanisms in Table~\ref{tab:missing_mechanism_summary_out}, and the precision and recall of \algname{} on the MIMIC-IV-ED dataset under MCAR in Table~\ref{tab:mimic4ed_missing_summary_pr}.

\begin{table}[h!]
\centering
\caption{Summary of final-iteration out-of-sample performance under different missing ratios and missing mechanisms.}
\label{tab:missing_mechanism_summary_out}
\scriptsize
\setlength{\tabcolsep}{4.2pt}
\renewcommand{\arraystretch}{0.95}
\begin{adjustbox}{max width=\textwidth}
\begin{tabular}{l c ccc ccc ccc}
\toprule
\multirow{2}{*}{Method}
& \multirow{2}{*}{Ratio}
& \multicolumn{3}{c}{MCAR}
& \multicolumn{3}{c}{MAR}
& \multicolumn{3}{c}{MNAR} \\
\cmidrule(lr){3-5}\cmidrule(lr){6-8}\cmidrule(lr){9-11}
& 
& R$_{\mathrm{out}} \uparrow$ & P$_{\mathrm{out}} \uparrow$ & ACC$_{\mathrm{out}} \uparrow$
& R$_{\mathrm{out}} \uparrow$ & P$_{\mathrm{out}} \uparrow$ & ACC$_{\mathrm{out}} \uparrow$
& R$_{\mathrm{out}} \uparrow$ & P$_{\mathrm{out}} \uparrow$ & ACC$_{\mathrm{out}} \uparrow$ \\
\midrule
\algname{} & 10
& 90.86 & 73.62 & 78.34
& 87.25 & 80.39 & 75.32
& 90.24 & 71.67 & 72.51 \\
\algname{} & 20
& 82.39 & 58.67 & 69.35
& 88.09 & 65.42 & 66.16
& 86.03 & 58.76 & 65.85 \\
\algname{} & 30
& 78.69 & 49.40 & 63.65
& 65.13 & 43.98 & 57.06
& 76.66 & 55.00 & 67.91 \\
\algname{} & 40
& 76.79 & 45.88 & 65.92
& 74.60 & 57.59 & 67.99
& 77.14 & 41.34 & 60.38 \\
\bottomrule
\end{tabular}
\end{adjustbox}
\end{table}

\begin{table}[ht!]
\centering
\caption{Precision and recall of \algname{} on MIMIC-IV-ED under ordinary MCAR.}
\label{tab:mimic4ed_missing_summary_pr}
\setlength{\tabcolsep}{6pt}
\scriptsize
\begin{tabular}{lccc}
\toprule
Method & Ratio & R$_{\text{out}} \uparrow$ & P$_{\text{out}} \uparrow$ \\
\midrule
\algname{} & 10 & ${61.06\% \pm 4.82\%}$ & ${71.99\% \pm 0.67\%}$ \\
\algname{} & 20 & ${55.34\% \pm 6.69\%}$ & ${56.49\% \pm 0.81\%}$ \\
\algname{} & 30 & ${59.90\% \pm 6.62\%}$ & ${44.82\% \pm 2.58\%}$ \\
\algname{} & 40 & ${65.12\% \pm 4.19\%}$ & ${34.61\% \pm 1.24\%}$ \\
\bottomrule
\end{tabular}
\end{table}

\subsection{Ablation Studies}
\label{app:ablation}

\paragraph{Effect of Iterative Latent-State Refinement.} \label{app:ablation-em} 


We first study whether the refinement loop is necessary. Figure~\ref{fig:em-iterations} reports performance as a function of the number of refinement iterations. Across both datasets, most improvements occur in the first few refinement iterations. MM-detection and final-state accuracy typically improve substantially from the first iteration to the third iteration, while observation-induced-missing-value imputation error stabilizes after a small number of iterations. We therefore use a fixed small number of refinement
iterations in the main experiments.
\begin{figure}[h!]
    \centering
    \begin{minipage}{\linewidth}
        \centering
        \includegraphics[width=\linewidth]{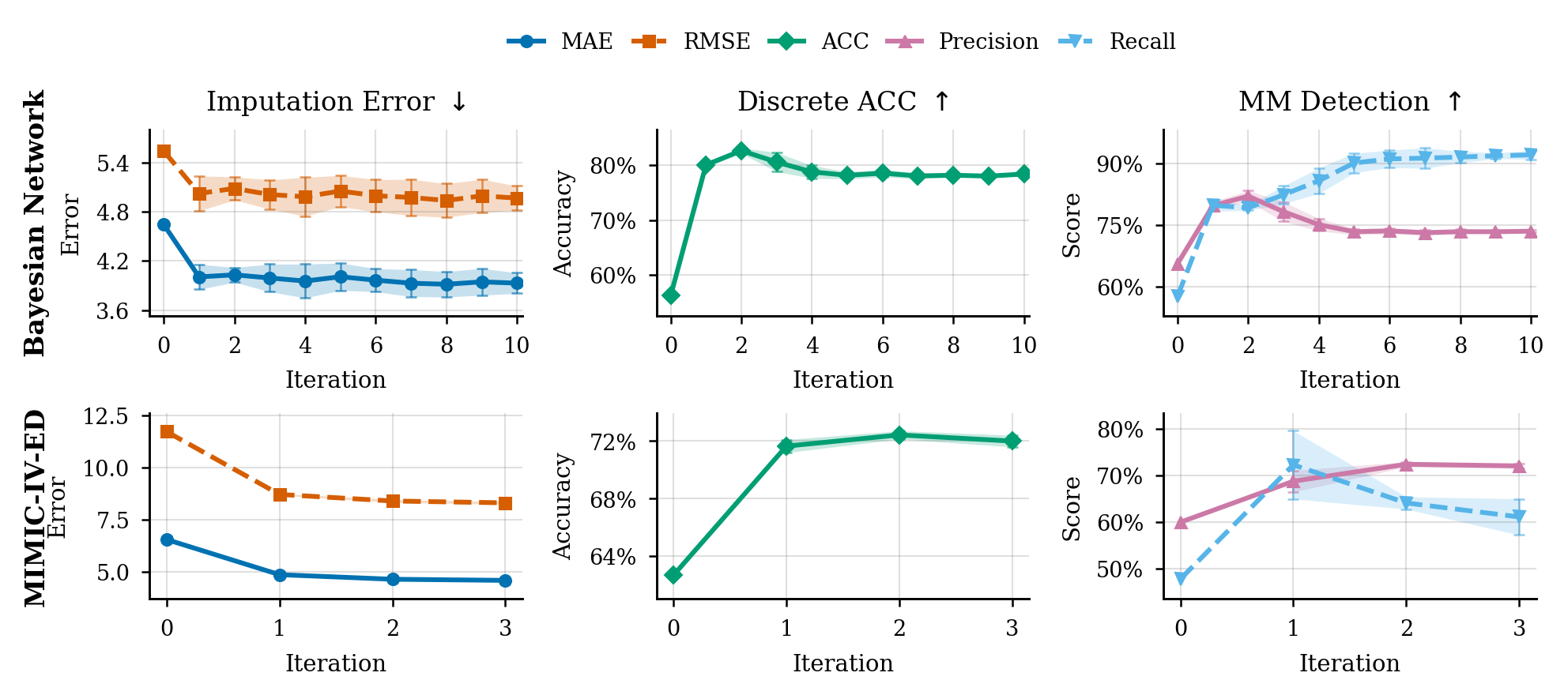}
        \caption{Effect of iterative latent-state refinement on Bayesian Network and MIMIC-IV-ED under ordinary MCAR at 10\% missing ratio. Curves show the mean over five seeds, with shaded regions indicating one standard deviation. Results under the zeroth iteration are obtained with a randomly initialized denoising network, with all other components kept the same as in \algname{}.}
        \label{fig:em-iterations}
    \end{minipage}
\end{figure}

\paragraph{Component Ablation.}
We ablate the main components of Diff-Joint on the Bayesian Network dataset under ordinary MCAR. DiffPuter is a standard diffusion imputer and does not output MM labels, whereas Diff-Joint models the joint state $(\bx,\bc)$ to capture dependencies between tabular values and MM labels. We further compare a one-refinement variant with the full model, which uses 10 outer refinement iterations. Table~\ref{tab:component_ablation} shows that joint $(\bx,\bc)$ modeling substantially improves final-state recovery over standard diffusion imputation, and iterative refinement further improves MM F1, final-state accuracy, and RMSE.

\begin{table}[ht!]
\centering
\caption{Component ablation on the Bayesian Network dataset under ordinary MCAR.}
\label{tab:component_ablation}
\setlength{\tabcolsep}{4.5pt}
\scriptsize
\begin{tabular}{llccccc}
\toprule
Missing Ratio & Method & Joint $(\bx,\bc)$ & Refinement & MM F1$_{\text{out}} \uparrow$ & Acc$_{\text{out}}$. $\uparrow$ & RMSE$_{\text{out}} \downarrow$ \\
\midrule
20 & DiffPuter & \xmark & \cmark & -- & 44.61\% & 5.1962 \\
20 & \algname, 1-refine & \cmark & \xmark & {69.85\%} & {71.11\%} & 5.1371 \\
20 & \algname, 10-iterations & \cmark & \cmark & \textbf{73.39\%} & \textbf{76.18\%} & \textbf{5.0832} \\
\midrule
30 & DiffPuter & \xmark & \cmark & -- & 51.89\% & 5.2764 \\
30 & \algname, 1-refine & \cmark & \xmark & 58.99\% & 60.26\% & 5.1461 \\
30 & \algname, 10-iterations & \cmark & \cmark & \textbf{64.76\%} & \textbf{72.67\%} & \textbf{4.9987} \\
\bottomrule
\end{tabular}
\end{table}

\paragraph{Aggregation-Rule Ablation.} \label{app:ablation-update-rule}

We conduct an ablation study on the rule used to update the meaningful-missingness mask during iterative latent-state refinement. Let $U$ denote the uncertainty-based signal obtained from the high-uncertainty cluster, and let $V$ denote the sampled-mask signal obtained from the joint diffusion model.
Table~\ref{tab:update_rule_ablation_closedloop_ratio30} reports the aggregation-rule ablation results. 
Majority-vote-only collapses from the all-zero MM-mask initialization,
and the AND rule inherits this conservativeness. Uncertainty-only provides a nontrivial bootstrap signal, while OR fusion achieves the best MM F1 and
final-state accuracy among the tested closed-loop rules. This supports using OR fusion as the default refinement rule, as it achieves the best overall performance in accurately identifying meaningful missingness.

\begin{table}[ht!]
\centering
\caption{Aggregation-rule ablation on the Bayesian Network dataset under ordinary MCAR at missing ratio 30.}
\label{tab:update_rule_ablation_closedloop_ratio30}
\setlength{\tabcolsep}{3.5pt}
\renewcommand{\arraystretch}{0.95}
\scriptsize
\begin{adjustbox}{max width=\textwidth}
\begin{tabular}{lccccc}
\toprule
Rule 
& Recall$_{\mathrm{out}}\uparrow$
& Precision$_{\mathrm{out}}\uparrow$
& MM F1$_{\mathrm{out}}\uparrow$
& ACC$_{\mathrm{out}}\uparrow$
& RMSE$_{\mathrm{out}}\downarrow$ \\
\midrule
Uncertainty-only
& {66.83\%} & {46.14\%} & {54.59\%} & {56.42\%} & 5.026 \\
Majority-vote-only
& 0.00\% & 0.00\% & 0.00\% & 51.27\% & {5.006} \\
Uncertainty $\wedge$ majority vote
& 0.00\% & 0.00\% & 0.00\% & 51.32\% & \textbf{4.942} \\
Uncertainty $\vee$ majority vote (ours)
& \textbf{76.33\%} & \textbf{54.08\%} & \textbf{63.31\%}
& \textbf{66.95\%} & 5.168 \\
\bottomrule
\end{tabular}
\end{adjustbox}
\end{table}

\begin{remark}
    The results in Tables~\ref{tab:component_ablation} and~\ref{tab:update_rule_ablation_closedloop_ratio30} are based on a single randomized run for each
configuration and are intended mainly for qualitative comparison among variants.
Due to random-seed variation, the absolute values are not expected to exactly
match the corresponding main-table results.
\end{remark}

\section{More Implementation Details on Baseline methods}
\label{app:exp-details-baselines}

We compare against recent deep-learning baselines for tabular imputation.
Specifically, we include CACTI~\cite{gorla2025cacti}, a recent strong masked-autoencoding method whose original benchmark reports improvements over autoencoding baselines such as ReMasker~\cite{du2024remasker} and AutoComplete~\cite{an2023autocomplete}. We also include DiffPuter~\cite{zhang2025diffputer}, a recent diffusion-based imputation method that outperforms prior diffusion baselines such as TabCSDI~\cite{zheng2022tabcsdi} and MissDiff~\cite{ouyang2023missdiff} on standard imputation benchmarks. On the Bayesian Network dataset, whose column names carry no semantic information, we use the non-embedding variant of CACTI, denoted CMAE.

All baselines are trained and evaluated on exactly the same observed datasets and
the same train/test splits as \algname{}. When compatible, we use the same preprocessing pipeline as in the main experiments; otherwise, we follow the
model-specific preprocessing required by the official implementation. 
Since these baselines are not designed to explicitly infer meaningful-missingness labels, MM precision and
recall are not applicable and are therefore omitted for them.

\paragraph{Implementation Sources.} We implement the baseline methods according to the following description or publicly available
codebases.
\begin{itemize}[leftmargin=2em, itemsep=2pt, topsep=0pt, parsep=0pt, partopsep=0pt]
    \item \textbf{Mean/Mode}:  Missing continuous entries are filled with the empirical mean of the corresponding column, and missing discrete entries are filled with the empirical mode. Both statistics are computed from observed entries in the training split and then reused for test-set imputation. This baseline has no tunable hyperparameters.
    \item\textbf{missForest} \cite{stekhoven2012missforest}: We use the implementations provided in the official HyperImpute \cite{pmlr-v162-jarrett22a} repository: \url{https://github.com/vanderschaarlab/hyperimpute}.
    \item \textbf{CACTI/CMAE} \cite{gorla2025cacti}: We use the official implementation of CACTI at \url{https://github.com/sriramlab/CACTI}. On MIMIC-IV-ED, we run CACTI with contextual column-name embeddings. Specifically, we precompute feature-name embeddings using \path{sentence-transformers/all-MiniLM-L6-v2} \cite{wang2020minilm} and use the resulting \texttt{embeddings\_colnames\_MiniLM.npz} file as CACTI's column-context input. On the Bayesian Network dataset, the column names are synthetic and carry no semantic information. We therefore disable the column-name embedding module and denote the resulting non-contextual variant as CMAE. This change removes arbitrary column-name semantics while keeping the masked autoencoding architecture and the remaining model configuration unchanged.
    
    \item\textbf{DiffPuter} \cite{zhang2025diffputer}:
    We use the official implementation at
\url{https://github.com/hengruizhang98/DiffPuter}.
\end{itemize}

\paragraph{Hyperparameter 
Settings.}
We use a fixed hyperparameter protocol for all baseline methods. Hyperparameters
are selected, when needed, using only the training split and an internal
validation configuration. The validation criterion is ordinary-value imputation
performance on held-out validation entries. We do not tune any method on the
test set, downstream prediction performance, or MM precision/recall/F1.

Mean/Mode has no tunable hyperparameters. For CACTI/CMAE and DiffPuter, we
follow the recommended or default configurations from the corresponding official
implementations, except for dataset-specific changes required by input
dimensionality, memory constraints, or contextual inputs. For missForest, we use
the \texttt{hyperimpute} search space and perform hyperparameter selection
separately for each random seed, using only the corresponding training and
validation split. The detailed baseline configurations are listed in
Table~\ref{tab:baseline_hyperparameters}. Final results are summarized as the
mean and standard deviation over five random seeds.

\begin{table}[h]
\centering
\caption{Hyperparameter settings used for the baseline methods in the main experiments.}
\label{tab:baseline_hyperparameters}
\setlength{\tabcolsep}{5pt}
\renewcommand{\arraystretch}{1.15}
\scriptsize
\begin{tabular}{p{0.18\linewidth}p{0.76\linewidth}}
\toprule
\textsc{Model} & \textsc{Hyperparameters / configuration} \\
\midrule

\textsc{Mean/Mode}
&
No tunable hyperparameters.\\
\midrule

\textsc{missForest}
&
Implemented using the \texttt{hyperimpute} package. For each random seed, we sample
50 hyperparameter configurations from the package-provided search space and report
the best-performing configuration for that seed. Final results are summarized as
mean and standard deviation over five seeds. \\
\midrule

\textsc{CMAE}
&
Used on the Bayesian Network dataset. CMAE uses the CACTI masked-autoencoding
architecture with the column-name embedding module disabled, because the synthetic column names carry no semantic information. The optimization settings
are \texttt{epochs = 300}, \texttt{warmup\_epochs = 50},
\texttt{batch\_size = 128}, \texttt{lr = 1e-3}, \texttt{min\_lr = 5e-6},
\texttt{weight\_decay = 1e-3}, \texttt{grad\_clip = 5.0},
\texttt{mask\_ratio = 0.9}, \texttt{embed\_dim = 64},
\texttt{nencoder = 10}, and \texttt{ndecoder = 4}. \\
\midrule

\textsc{CACTI}
&
Used on MIMIC-IV-ED. CACTI uses the same optimization settings as CMAE:
\texttt{epochs = 300}, \texttt{warmup\_epochs = 50},
\texttt{batch\_size = 128}, \texttt{lr = 1e-3}, \texttt{min\_lr = 5e-6},
\texttt{weight\_decay = 1e-3}, \texttt{grad\_clip = 5.0},
\texttt{mask\_ratio = 0.9}, \texttt{embed\_dim = 64},
\texttt{nencoder = 10}, and \texttt{ndecoder = 4}. Column-name embeddings are
precomputed using
\texttt{sentence-transformers/all-MiniLM-L6-v2}~\cite{wang2020minilm}.\\
\midrule

\textsc{DiffPuter}
&
Bayesian Network: \texttt{max\_iter = 10}. MIMIC-IV-ED:
\texttt{max\_iter = 10}. Common settings are \texttt{hid\_dim = 1024},
\texttt{num\_trials = 10}, \texttt{num\_steps = 50},
\texttt{num\_epochs = 1000}, \texttt{batch\_size = 4096},
\texttt{learning\_rate = 1e-4}, \texttt{scheduler\_patience = 40}, and
\texttt{early\_stop\_patience = 200}. \\

\bottomrule
\end{tabular}
\end{table}

\end{document}